\documentclass{article}

\usepackage{microtype}
\usepackage{graphicx}
\usepackage{subfigure}
\usepackage{booktabs} 

\usepackage{hyperref}

\usepackage{xcolor}
\usepackage{amsmath}
\usepackage{amsfonts}
\usepackage{amsthm}
\usepackage{multirow}
\usepackage{tabularx}
\usepackage[skip=10pt]{caption}
\usepackage{paralist}
\usepackage[hang,flushmargin]{footmisc}

\newcommand{\alex}[1]{\textcolor[HTML]{D35400}{[\textbf{AK:} {\em #1}]}}
\newcommand{\woj}[1]{\textcolor[RGB]{76,0,153}{[\textbf{WM:} {\em #1}]}}
\newcommand{\oh}[1]{\textcolor[RGB]{153,0,76}{[\textbf{TH:} {\em #1}]}}
\newcommand{\liane}[1]{\textcolor{blue}{[\textbf{LM:} {\em #1}]}}
\newcommand{\petr}[1]{\textcolor[RGB]{15,120,30}{[\textbf{Petr:} {\em #1}]}}
\newcommand{\ignore}[1]{}

\renewcommand{\alex}[1]{}
\renewcommand{\woj}[1]{}
\renewcommand{\oh}[1]{}
\renewcommand{\liane}[1]{}
\renewcommand{\petr}[1]{}

\newcommand{\calD}{{\mathcal{D}}}

\newcommand{\calG}{{\mathcal{G}}}
\newcommand{\calH}{{\mathcal{H}}}

\newcommand{\calL}{{\mathcal{L}}}

\newcommand{\calS}{{\mathcal{S}}}

\newcommand{\calX}{{\mathcal{X}}}
\newcommand{\calY}{{\mathcal{Y}}}

\newcommand{\Real}{\mathbb R}

\newcommand{\be}{\begin{eqnarray}}
\newcommand{\ee}{\end{eqnarray}}
\newcommand{\bee}{\begin{eqnarray*}}
\newcommand{\eee}{\end{eqnarray*}}

\newcommand{\matrixb}{\left[ \begin{array}}
\newcommand{\matrixe}{\end{array} \right]}   

\newtheorem{theorem}{{Theorem}}
\newtheorem{lemma}{{Lemma}}

\newtheorem{definition}{{Definition}}

\newcommand{\argmax}{\operatornamewithlimits{\arg \max}}
\newcommand{\argmin}{\operatornamewithlimits{\arg \min}}

\newcommand{\E}{\mathbb E}

\usepackage{xspace}
\makeatletter
\DeclareRobustCommand\onedot{\futurelet\@let@token\@onedot}
\def\@onedot{\ifx\@let@token.\else.\null\fi\xspace}

\DeclareMathAlphabet\mathbfcal{OMS}{cmsy}{b}{n}

\def\eg{e.g\onedot} 
\def\ie{i.e\onedot} 
 
\def\etc{etc\onedot} 
\def\wrt{w.r.t\onedot}

\def\Vec#1{\mbox{ $\vec{ \mathbf{ #1}}$}}

\newcommand{\Tref}[1]{Table~\ref{#1}}
\newcommand{\Eref}[1]{Eq.~(\ref{#1})}

\newcommand{\Cref}[1]{Chap.~\ref{#1}}

\renewcommand{\paragraph}[1]{\noindent\textbf{#1}\,\,\,}

\usepackage[accepted]{icml2019}

\icmltitlerunning{Neural Inverse Knitting: From Images to Manufacturing Instructions}

\begin{document}

\twocolumn[
\icmltitle{Neural Inverse Knitting: From Images to Manufacturing Instructions}



\icmlsetsymbol{equal}{*}

\begin{icmlauthorlist}
\icmlauthor{Alexandre Kaspar}{equal,mit}
\icmlauthor{Tae-Hyun Oh}{equal,mit}
\icmlauthor{Liane Makatura}{mit}
\\
\icmlauthor{Petr Kellnhofer}{mit}
\icmlauthor{Jacqueline Aslarus}{hs}
\icmlauthor{Wojciech Matusik}{mit}
\end{icmlauthorlist}

\icmlaffiliation{mit}{Computer Science \& Artificial Intelligence Laboratory (CSAIL), Massachusetts Institute of Technology (MIT), Cambridge, MA, USA}
\icmlaffiliation{hs}{Weston High School, Weston, MA, USA, work done during an internship at MIT}

\icmlcorrespondingauthor{Alexandre Kaspar}{akaspar@mit.edu}

\icmlkeywords{Machine Knitting, Program Synthesis, Machine Learning}

\vskip 0.3in
]



\renewcommand{\ICML@appearing}{
}

\printAffiliationsAndNotice{
\indent\phantom{\hspace{2.4mm}}{Project\,page: \url{http://deepknitting.csail.mit.edu}}\\
\icmlEqualContribution
} 

\begin{abstract}
Motivated by the recent potential of mass customization brought by whole-garment knitting machines, we introduce the new problem of automatic machine instruction generation using a single image of the desired physical product, which we apply to machine knitting.
We propose to tackle this problem by directly learning to synthesize regular machine instructions from real images.
We create a cured dataset of real samples with their instruction counterpart and propose to use synthetic images to augment it in a novel way.
We theoretically motivate our data mixing framework and show empirical results suggesting that making real images look more synthetic is beneficial in our problem setup.
\end{abstract}

\newcommand{\figTeaser}{
\begin{figure}[t]
\centering
\includegraphics[width=1\linewidth]{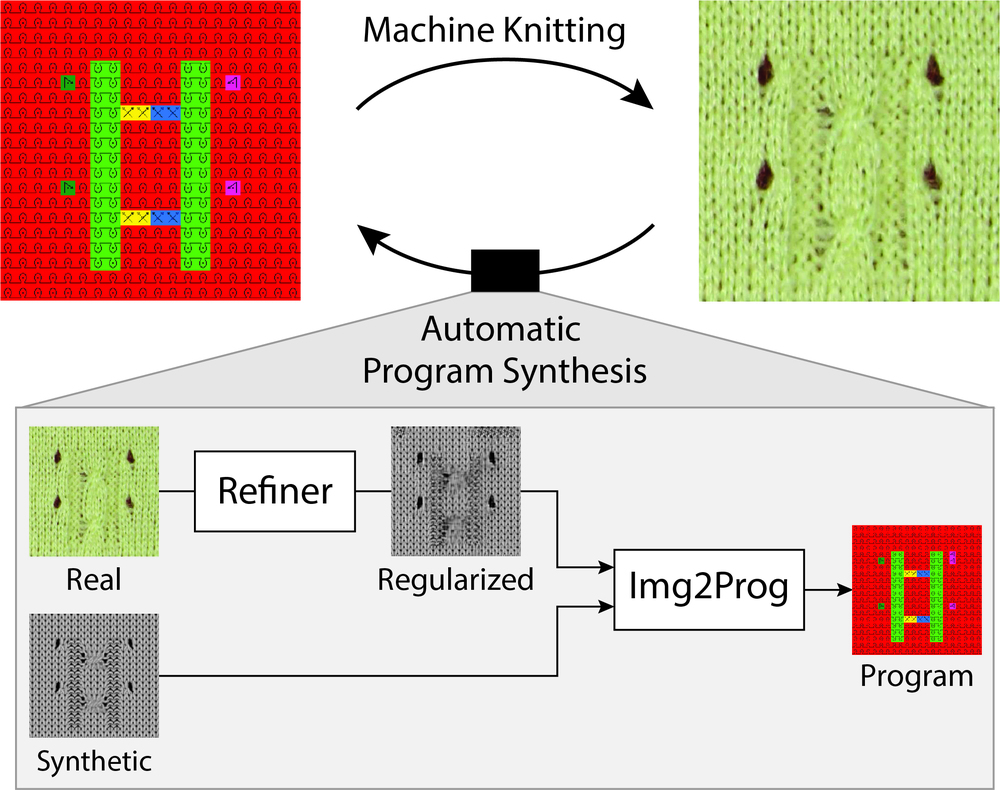}
\caption{
Illustration of our inverse problem and solution.
An instruction map (\emph{top-left}) is knitted into a physical artifact (\emph{top-right}).
We propose a machine learning pipeline to solve the inverse problem by leveraging synthetic renderings of the instruction maps.
}
\label{fig:teaser}
\end{figure}
}

\newcommand{\figMove}[1][t]{
\begin{figure*}[#1]
\centering
\includegraphics[width=\linewidth]{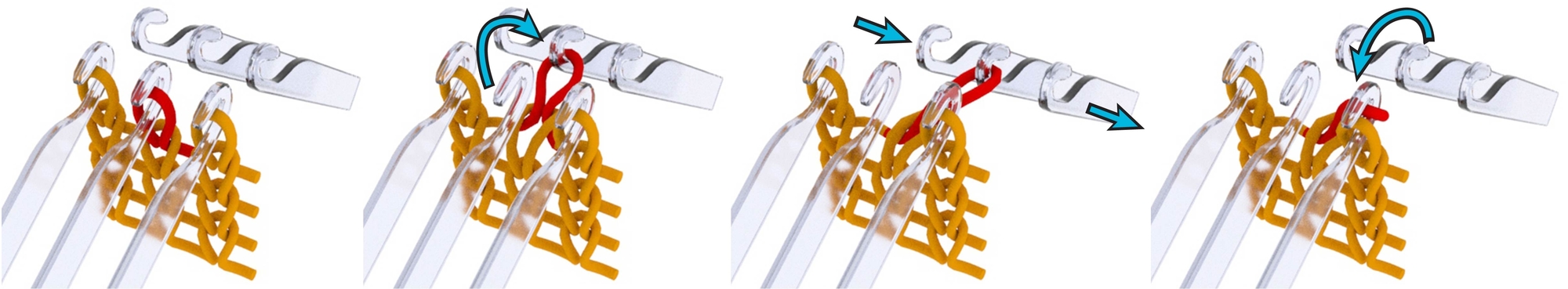}\vspace{-3mm}
\caption{Sample \emph{Transfer} sequence: move the red center stitch to the opposite bed; rack (move) the back bed 1 needle relative to the front; transfer the red stitch back to its original side. Note that the center front needle is now empty, while the right front needle holds 2 stitches.}
\label{fig:knit_miss_tuck}
\end{figure*}
}

\newcommand{\figKnitMissTuck}[1][t] {
\begin{figure}[#1]
\centering
\includegraphics[width=\linewidth]{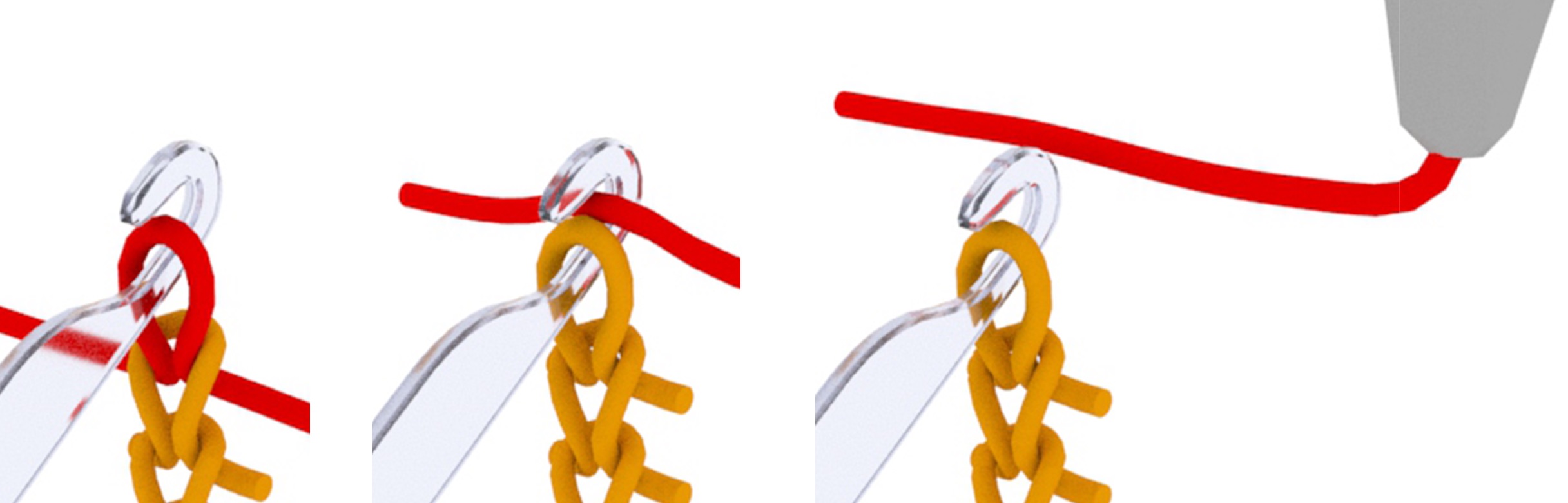}\vspace{-3mm}
\caption{(L to R) Illustration of \emph{Knit}, \emph{Tuck}, and \emph{Miss} operations.}
\label{fig:move}
\end{figure}
}

\newcolumntype{C}{>{\centering}X}%

\newcommand{\figInstructionSet}[1][t]{
\begin{figure*}[#1]
\centering
\includegraphics[width=\linewidth]{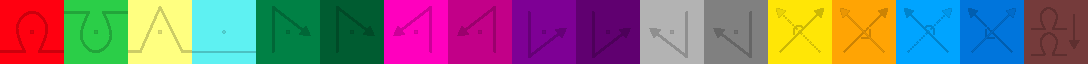}
\footnotesize
\begin{tabularx}{\linewidth}{CCCCCCCCCCCCCCCCC}
K & P & T & M & FR1 & FR2 & FL1 & FL2 & BR1 & BR2 & BL1 & BL2 & XR+ & XR- & XL+ & XL- & S
\end{tabularx}
\caption{
\textbf{Top}: abstract illustration and color coding of our $17$ instructions.
\textbf{Bottom}: instruction codes, which can be interpreted using the initial character of the following names:
\textbf{K}nit and \textbf{P}url (front and back knit stitches), 
\textbf{T}uck, \textbf{M}iss,
\textbf{F}ront, \textbf{B}ack,
\textbf{R}ight, \textbf{L}eft, 
\textbf{S}tack.
Finally, \textbf{X} stands for \emph{Cross} where $+$ and $-$ are the ordering (upper and lower).
\emph{Move} instructions are composed of their initial knitting side (\textbf{F}ront or \textbf{B}ack), the move direction (\textbf{L}eft or \textbf{R}ight) and the offset (1 or 2).
}
\label{fig:instruction_set}
\end{figure*}
}

\newcommand{\figDataset}[1][t]{
\begin{figure}[#1]
\centering
\includegraphics[width=\linewidth]{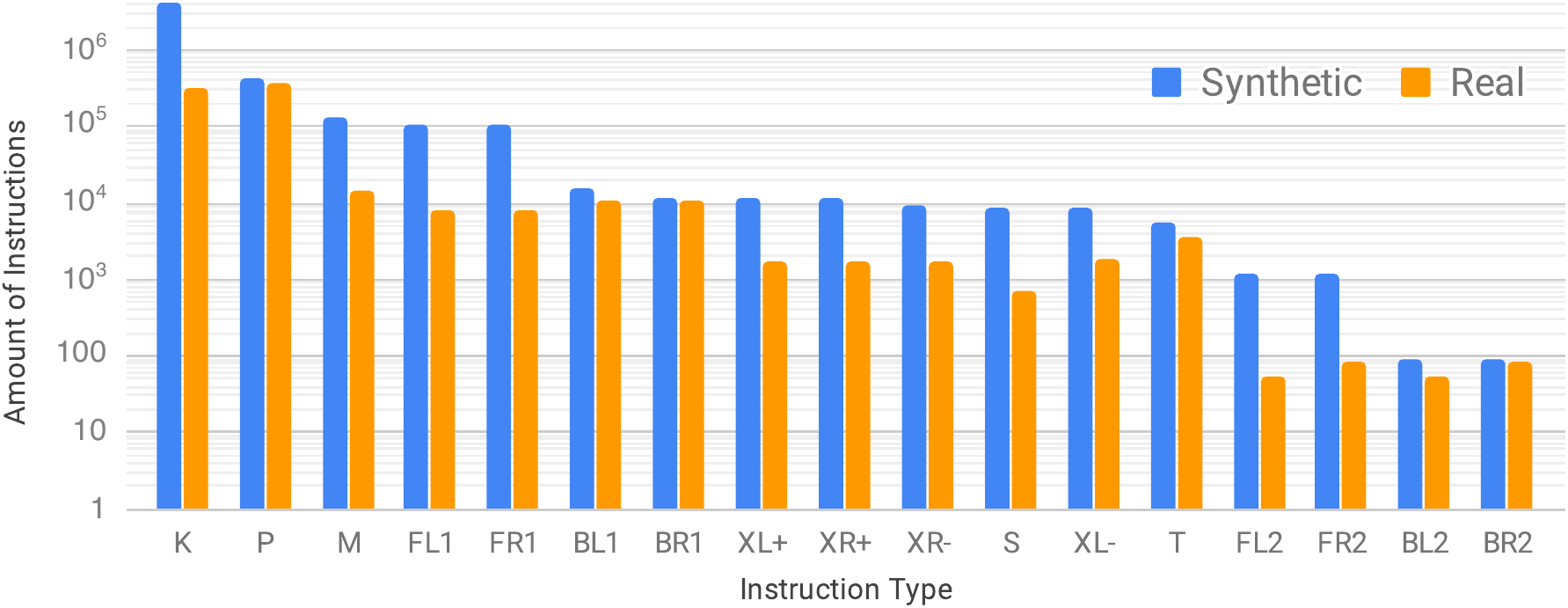}\vspace{-3mm}
\caption{
Instruction counts in descending order, for synthetic and real images.
Note the logarithmic scale of the Y axis.
}
\label{fig:dataset}
\end{figure}
}

\newcommand{\figSupervisedData}[1][t]{
\begin{figure}[#1]
\centering
\includegraphics[width=1\linewidth]{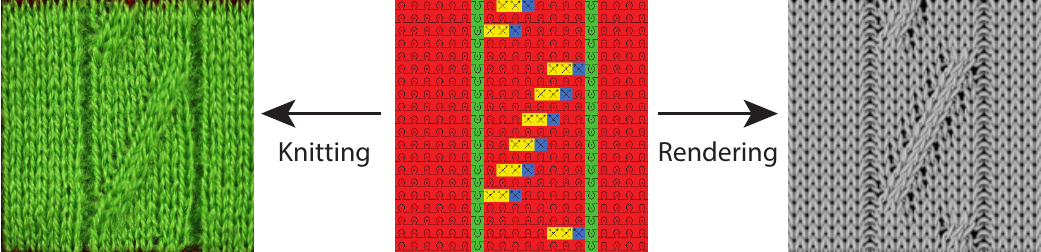}\vspace{-3mm}
\caption{
Different parts of our dataset (from left to right): real data images, machine instructions, and black-box rendering.
}
\label{fig:supervised_data}
\vspace{-3mm}
\end{figure}
}

\newcommand{\figAcquisition}[1][t]{
\begin{figure}[#1]
\centering
\includegraphics[height=3.3cm]{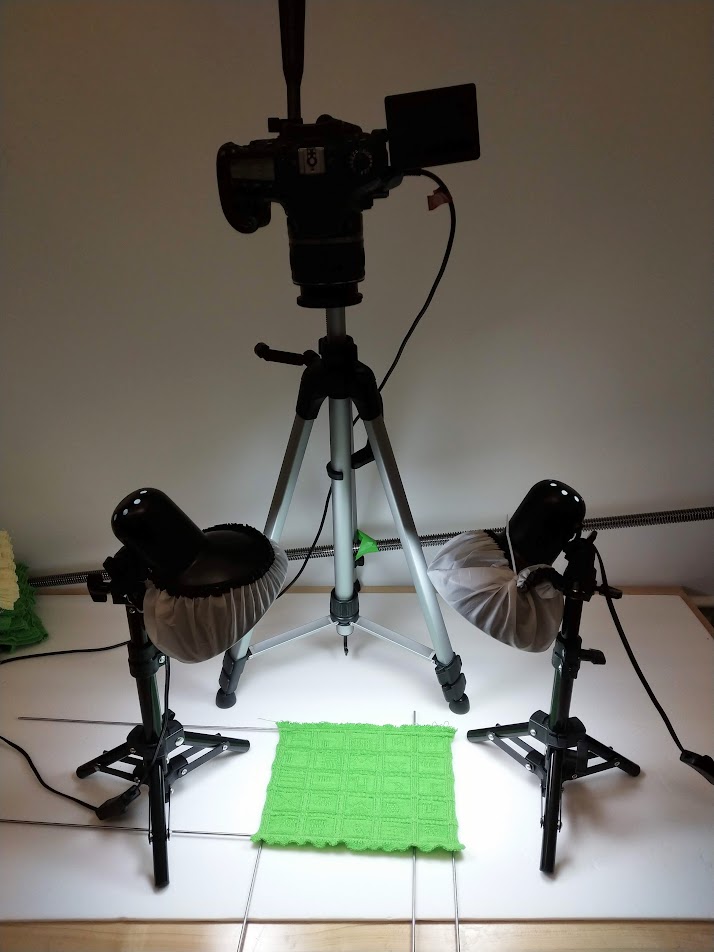}
\includegraphics[height=3.3cm]{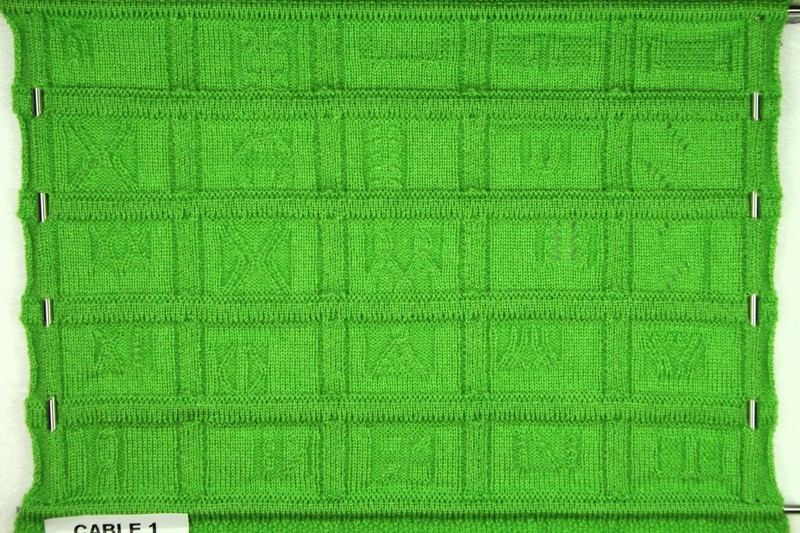}
\caption{
Our basic capture setup and a sample of $5\times 5$ knitted patterns with tension controlled by steel rods.
}
\label{fig:acquisition}
\end{figure}
}

\newcommand{\figBaselines}[1][ht!]{
\begin{figure*}[#1]
\includegraphics[width=\linewidth]{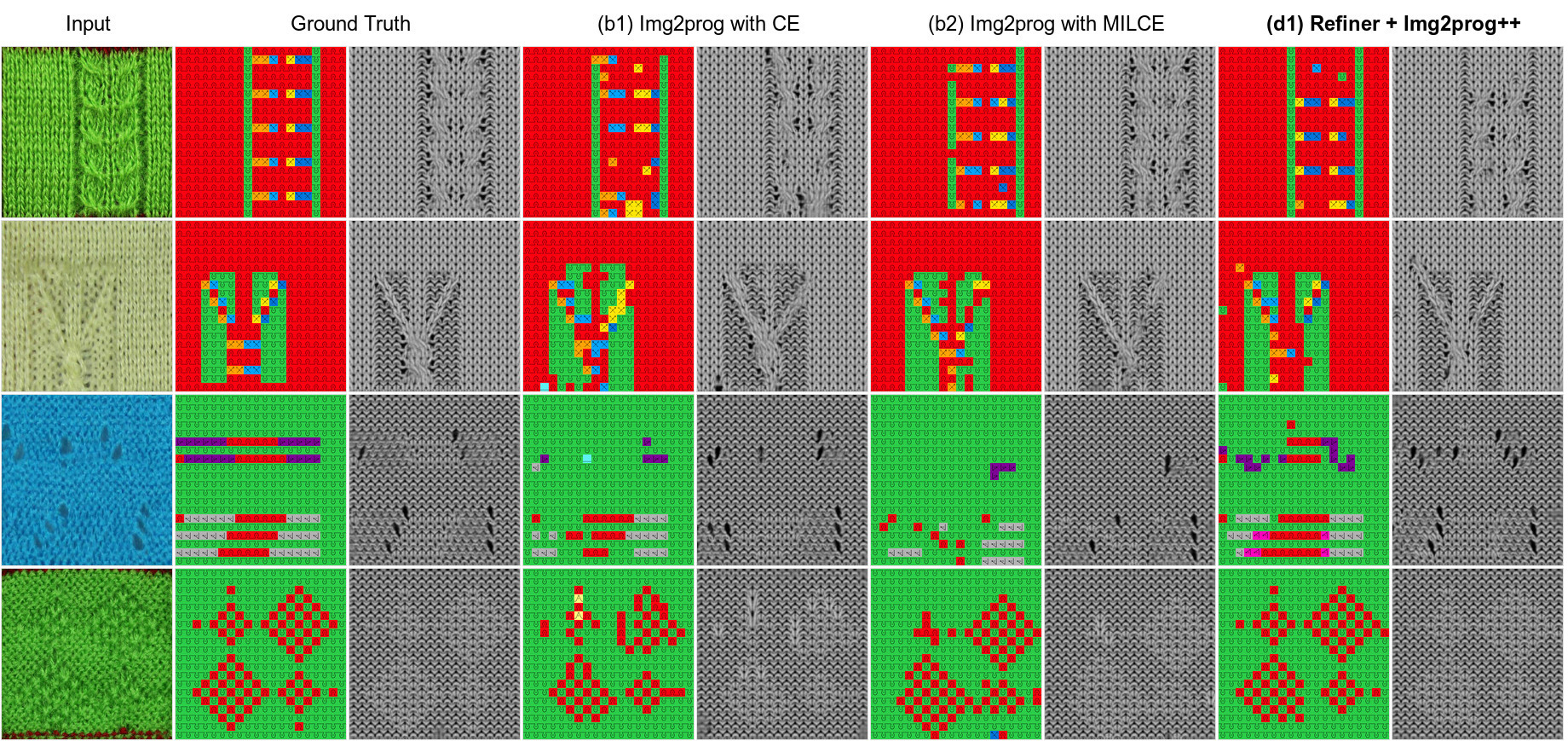}\vspace{-3mm}
\caption{
A comparison of instructions predicted by different versions of our method. We present the predicted instructions as well as a corresponding image from our renderer.
}
\label{fig:baselines}
\end{figure*}
}

\newcommand{\figBaselinesL}[1][ht!]{
\begin{figure*}[#1]
\includegraphics[width=\linewidth]{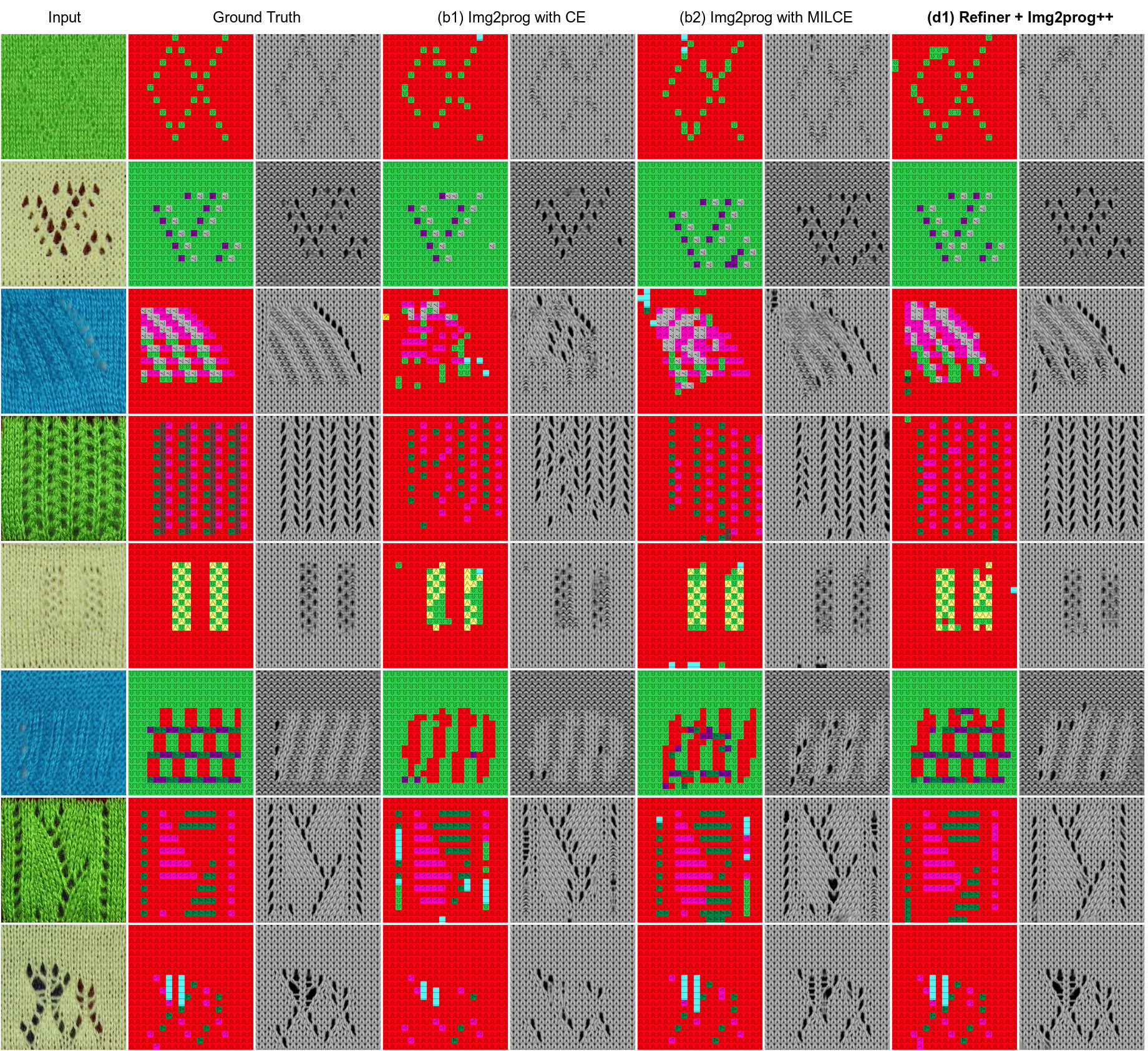}
\caption{
Additional comparisons of instructions predicted by different version of our method. We present the predicted instructions as well as a corresponding image from our renderer.
}
\label{fig:baselines_large}
\end{figure*}
}

\newcommand{\figDatasetAccuracy}[1][t]{
\begin{figure}[#1]
\includegraphics[width=\linewidth]{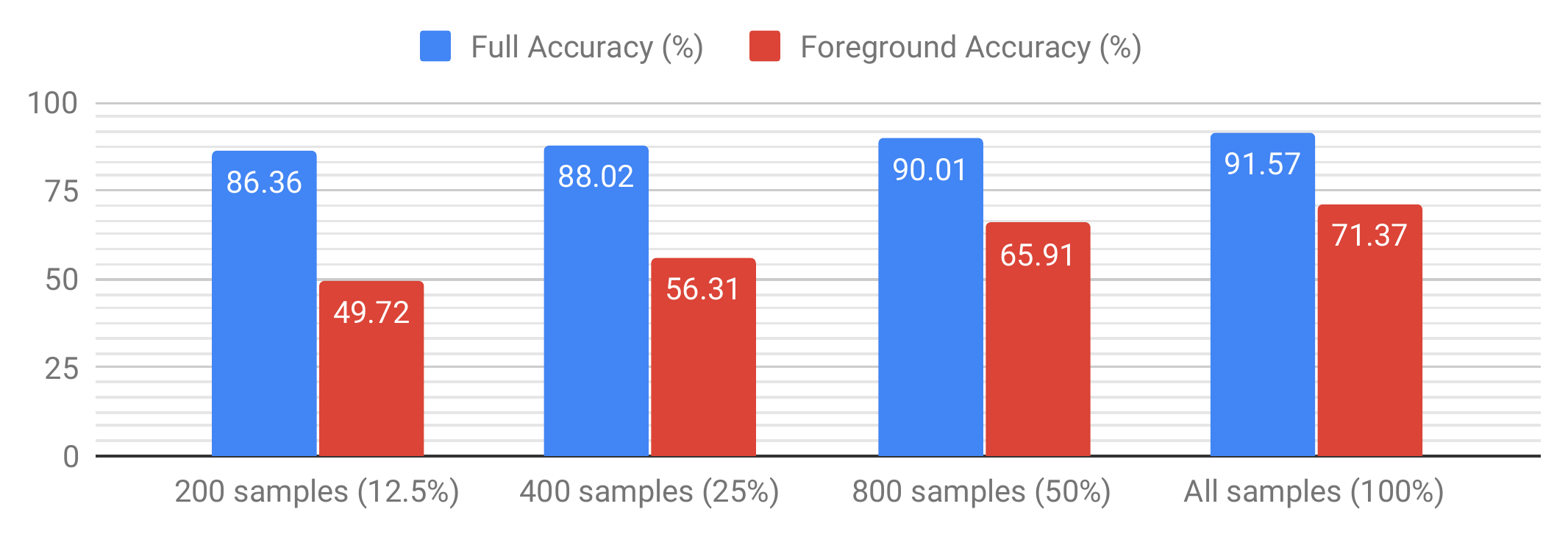}
\caption{
The impact of the amount of real training data (from 12.5\% to 100\% of the real dataset) over the accuracy.
}
\label{fig:dataset_accuracy}
\vspace{-3mm}
\end{figure}
}

\newcommand{\scale}[3]{
\begin{minipage}[t]{0.065\linewidth}
\centering
#2 #1 \\
\vspace{1mm}
\includegraphics[width=\linewidth]{figures/scale/#1.jpg} \\
\vspace{-3mm}
#3
\end{minipage}
}

\newcommand{\figScaleEstimationExp}[1][t]{
\definecolor{ScalePeak}{HTML}{EFEFEF}
\begin{figure*}[#1]
\centering
\resizebox{0.9\linewidth}{!}{
\scale{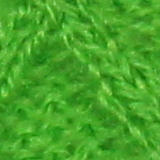}{}{}
\scale{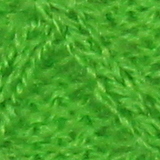}{}{}
\scale{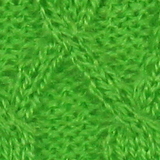}{}{}
\scale{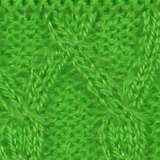}{}{}
\scale{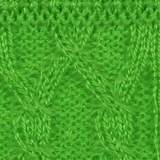}{}{}
\colorbox{ScalePeak}{
\scale{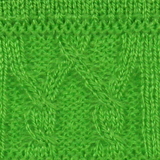}{\bf}{\noindent\rule{\textwidth}{1pt}}
}
\scale{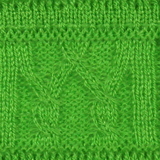}{}{}
\scale{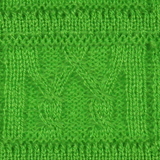}{}{}
\scale{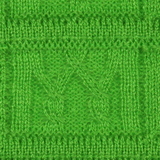}{}{}
\scale{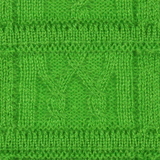}{}{}
\scale{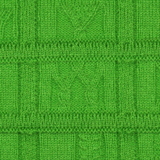}{}{}
\scale{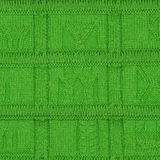}{}{}
}\\
\includegraphics[width=0.9\linewidth]{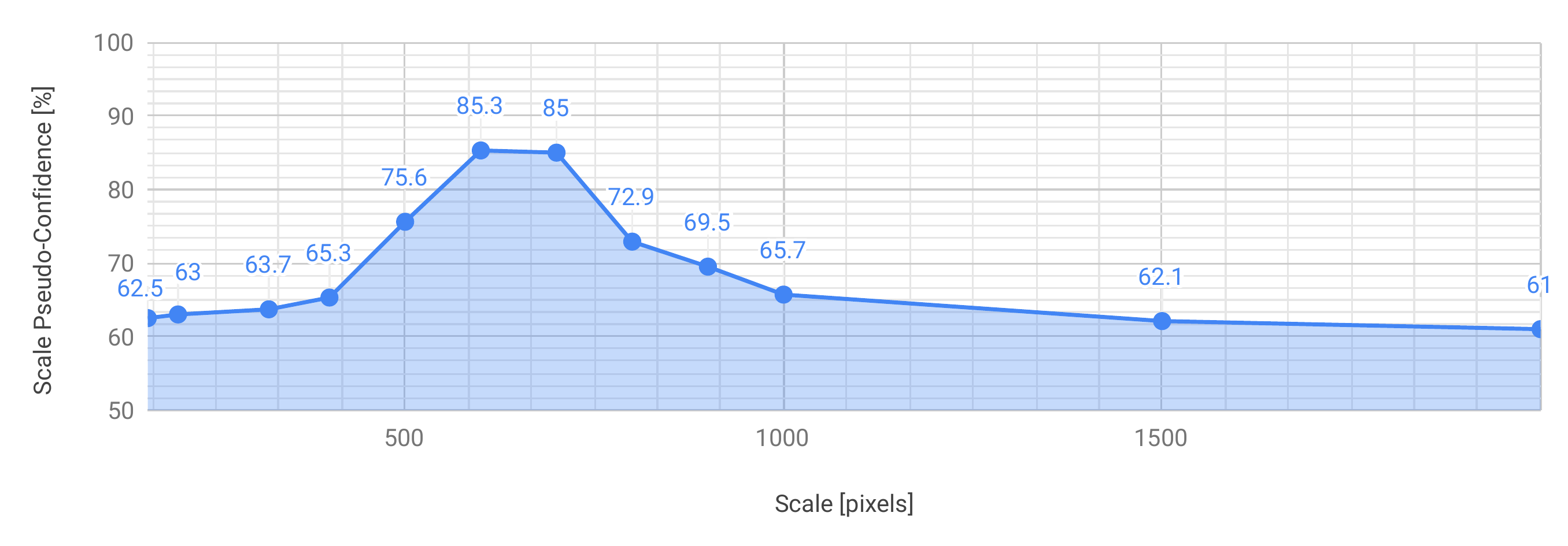}
\vspace{-3mm}
\caption{Scale identification experiment.
\textbf{Top row}: cropped input image at corresponding scales with the correct pixel scale in bold with a light-gray background.
\textbf{Plot}: pseudo-confidence curve showing a peak at the correct pixel scale ($600$).
}
\label{fig:scale}
\end{figure*}
}

\newcommand{\refSec}[1]{Section~\ref{sec:#1}}
\newcommand{\refSubSec}[1]{Section~\ref{subsec:#1}}
\newcommand{\refFig}[1]{Figure~\ref{fig:#1}}
\newcommand{\refEq}[1]{Equation~\ref{eq:#1}}
\newcommand{\refTbl}[1]{Table~\ref{tbl:#1}}
\newcommand{\tabComparisons}[1][t]{
\begin{table}[#1]
\vspace{-2mm}
\centering
\caption{\textbf{Performance comparison to baseline methods on our real image test dataset.} The table shows translation invariant accuracy of the predicted instructions with and without the background and PSNR and SSIM metrics for the image reconstruction where available. More is better for all metrics used.}
\vspace{5pt}
\resizebox{1.00\linewidth}{!}{
\begin{tabular}{@{\hskip0.5mm}c@{\hskip0.5mm}|l|c|c|c|c}
& \multirow{2}{*}{\textbf{Method}} & \multicolumn{2}{c|}{\textbf{Accuracy (\%)}} &
\multicolumn{2}{c}{\textbf{Perceptual}}  \\ 
\cline{3-4}
\cline{5-6}
& & \textbf{Full} & \textbf{FG} & \textbf{SSIM} & \textbf{PSNR [dB]} \\
\hline
(a1) & CycleGAN \cite{Zhu17}       & 57.27 & 24.10 & 0.670 & 15.87 \\
(a2) & Pix2Pix \cite{Isola16}      & 56.20 & 47.98 & 0.660 & 15.95 \\
(a3) & UNet \cite{Ronneberger15}   & 89.65 & 63.99 & 0.847 & 21.21 \\
(a4) & Scene Parsing \cite{Zhou16} & \bf 91.58 & \bf 73.95 & \bf 0.876 & \bf 22.64 \\
(a5) & S+U \cite{shrivastava2017learning} & 91.32 & 71.00 & 0.864 & 21.42 \\
\hline
(b1) & \texttt{Img2prog} (real only) with CE    & 91.57 & 71.37 & 0.866 & \bf 21.62 \\
(b2) & \texttt{Img2prog} (real only) with MILCE & \bf 91.74 & \bf 72.30 & \bf 0.871 & 21.58 \\
\hline
(c1) & Refiner + \texttt{Img2prog} ($\alpha=0.9$)   & 93.48 & 78.53 & 0.894 & \bf 23.28 \\
(c2) & Refiner + \texttt{Img2prog} ($\alpha=2/3$)   & \bf 93.58 & \bf 78.57 & 0.892 & 23.27 \\
(c3) & Refiner + \texttt{Img2prog} ($\alpha=0.5$)   & 93.57 & 78.30 & \bf 0.895 & 23.24 \\
(c4) & Refiner + \texttt{Img2prog} ($\alpha=1/3$)   & 93.19 & 77.80 & 0.888 & 22.72 \\
(c5) & Refiner + \texttt{Img2prog} ($\alpha=0.1$)   & 92.42 & 74.15 & 0.881 & 22.27 \\
\hline
(d1) & Refiner + \texttt{Img2prog}++ ($\alpha=0.5$) & \bf 94.01 & \bf 80.30 & \bf 0.899 & \bf 23.56 \\
\end{tabular}
}
\label{tab:comparisons}
\vspace{-3mm}
\end{table}
}

\newcommand{\tableLargeNetwork}[1][t]{
\begin{table*}[#1]
\centering
\caption{\textbf{Performance comparison with larger scene parsing network from~\cite{Zhou16}}. (d2) uses pre-training on ImageNet~\cite{russakovsky2015imagenet} and a much larger number of parameters (1.4M v.s. 51.4M).}
\label{tab:large}
\vspace{5pt}
\resizebox{0.75\linewidth}{!}{
\begin{tabular}{@{\hskip0.5mm}c@{\hskip0.5mm}|l|c|c|c|c|c}
& \multirow{2}{*}{\textbf{Method}} & \multicolumn{2}{c|}{\textbf{Accuracy (\%)}} &
\multicolumn{2}{c|}{\textbf{Perceptual}} &
\textbf{\# Parameters} \\ 
\cline{3-4}
\cline{5-6}
& & \textbf{Full} & \textbf{FG} & \textbf{SSIM} & \textbf{PSNR [dB]} & \textbf{(in Millions)}\\
\hline
(d1) & Refiner + img2prog++ ($\alpha=1/2$) & 94.01 & 80.30 & 0.899 & 23.56 & 1.4 \\
(d2) & Large Scene Parsing w/ pre-training & \bf 94.95 & \bf 83.46 & \bf 0.908 & \bf 24.58 & 51.4 \\
\end{tabular}
}
\end{table*}
}

\newcommand{\tabPerInstruction}[1][t]{
\begin{table*}[#1]
\vspace{-2mm}
\centering
\caption{
\textbf{Performance of \emph{Refined+Img2prog++} measured per instruction over the test set.
}
This shows that even though our instruction distribution has very large variations, our network is still capable of learning some representation for the least frequent instructions (3 orders of magnitude difference for FR2, FL2, BR2, BL2 compared to K and P).
\vspace{5pt}
}
\resizebox{1.00\linewidth}{!}{
\begin{tabular}{@{\hskip0.5mm}c|cccccccccccccccccc}
Instruction & K & P & T & M & FR1 & FR2 & FL1 & FL2 & BR1 & BR2 & BL1 & BL2 & XR+ & XR- & XL+ & XL- & S \\
\hline
Accuracy [\%] & 96.52 & 96.64 & 74.63 & 66.65 & 77.16 & 100.00 & 74.20 & 83.33 & 68.73 & 27.27 & 69.94 & 22.73 & 60.15 & 62.33 & 60.81 & 62.11 & 25.85 \\
Frequency [\%] & 44.39 & 47.72 & 0.41 & 1.49 & 1.16 & 0.01 & 1.23 & 0.01 &  1.22 & 0.02 & 1.40 & 0.02 & 0.22 & 0.18 & 0.19 & 0.22 & 0.12 \\
\end{tabular}
}
\label{tab:per_instruction}
\vspace{-3mm}
\end{table*}
}

\section{Introduction}
Advanced manufacturing methods that allow completely automated production of customized objects and parts are transforming today's economy. One prime example of these methods is whole-garment knitting that is used to mass-produce many common textile products (e.g., socks, gloves, sportswear, shoes, car seats, etc.). During its operation, a whole garment knitting machine executes a custom low-level program to manufacture each textile object. Typically, generating the code corresponding to each design is a difficult and tedious process requiring expert knowledge. A few recent works have tackled the digital design workflow for whole-garment knitting~\cite{Underwood09, McCann16, Narayanan18,Yuksel2012,Wu18a,Wu18b}. None of these works, however, provide an easy way to specify patterns.   

\figTeaser
\figMove
\figKnitMissTuck

The importance of patterning in textile design is evident in pattern books \cite{Donohue15, Shida17}, which contain instructions for hundreds of decorative designs that have been manually crafted and tested over time.
Unfortunately, these pattern libraries are geared towards hand-knitting and they are often incompatible with the operations of industrial knitting machines.
Even in cases when a direct translation is possible, the patterns are only specified in stitch-level operation sequences. Hence, they would have to be manually specified and tested for each machine type similarly to low-level assembly programming.

In this work, we propose an inverse design method using deep learning to automate the pattern design for industrial knitting machines. In our \emph{inverse knitting}, machine instructions are directly inferred from an image of the fabric pattern.
To this end, we collect a paired dataset of knitting instruction maps and corresponding images of knitted patterns.
We augment this dataset with synthetically generated pairs obtained using a knitting simulator~\cite{Shima11}.
This combined dataset facilitates a learning-based approach.
More specifically, we propose a theoretically inspired image-to-program map synthesis method that leverages both real and simulated data for learning.
Our contributions include:
{\setlength{\leftmargini}{6mm}
\begin{itemize}
\setlength\itemsep{0mm}
    \item An automatic translation of images to sequential instructions for a real manufacturing process;
    \item A diverse knitting pattern dataset that provides a mapping between images and instruction programs specified using a new domain-specific language (DSL)~\cite{kant2018recent} that significantly simplifies low-level instructions and can be decoded without ambiguity;
    \item A theoretically inspired deep learning pipeline to tackle this inverse design problem; and
    \item A novel usage of synthetic data to learn to neutralize real-world, visual perturbations.
\end{itemize}
}
In the rest of the paper, we first provide the necessary background in machine knitting and explain our 2D regular instructions, we then go over our dataset acquisition, detail our learning pipeline making use of synthetic data, and finally go over our experiment results.

\section{Knitting Background}
\label{sec:background}
Knitting is one of the most common forms of textile manufacturing.
The type of knitting machine we are considering in this work is known as a \emph{V-bed} machine, which allows automatic knitting of whole garments.
This machine type uses two beds of individually controllable needles, both of which are oriented in an inverted \emph{V} shape allowing opposite needles to transfer loops between beds.
The basic operations are illustrated in Figures~\ref{fig:knit_miss_tuck} and ~\ref{fig:move}:

{\setlength{\leftmargini}{4mm}
\begin{itemize}
\setlength\itemsep{-0.5mm}
    \item \textbf{Knit} pulls a new loop of yarn through all current loops,
    \item \textbf{Tuck} stacks a new loop onto a needle,
    \item \textbf{Miss} skips a needle,
    \item \textbf{Transfer} moves a needle's content to the other bed,
    \item \textbf{Racking} changes the offset between the two beds.
\end{itemize}
}

Whole garments (e.g. socks, sweatshirts, hats) can be automatically manufactured by scheduling complex sequences of these basic operations~\cite{Underwood09, McCann16}.
Furthermore, this manufacturing process also enables complex surface texture and various types of patterns.
Our aim is to automatically generate machine instructions to reproduce any geometric pattern from a single close-up photograph (e.g. of your friend's garment collection).
To simplify the problem, we assume the input image only captures 2D patterning effects of flat fabric, and we disregard variations associated with the 3D shape of garments.

\section{Instruction Set}

\figInstructionSet

General knitting programs are sequences of operations which may not necessarily have a regular structure.
In order to make our inverse design process more tractable, we devise a set of $17$ pattern 
instructions (derived from a subset of the hundreds of instructions from~\cite{Shima11}).
These instructions include all basic knitting pattern operations and they are specified on a regular 2D grid that can be parsed and executed line-by-line.
We first detail these instructions and then explain how they are sequentially processed.

The first group of instructions are based on the first three operations, namely: \textbf{Knit}, \textbf{Tuck} and \textbf{Miss}.

Then, \emph{transfer} operations allow moving loops of yarn across beds.
This is important because knitting on the opposite side produces a very distinct stitch appearance known as reverse stitch or \textbf{Purl} -- our complement instruction of \emph{Knit}.

Furthermore, the combination of transfers with \emph{racking} allows moving loops within a bed.
We separate such higher-level operations into two groups:
\textbf{Move} instructions only consider combinations that do not cross other such instructions so that their relative scheduling does not matter, and
\textbf{Cross} instructions are done in pairs so that both sides are swapped, producing what is known as \emph{cable} patterns.
The scheduling of \emph{cross} instructions is naturally defined by the instructions themselves.
These combined operations do not create any new loop by themselves, and thus we assume they all apply a \emph{Knit} operation before executing the associated needle moves, so as to maintain spatial regularity.

Finally, transfers also allow different stacking orders when multiple loops are joined together. We model this with our final \textbf{Stack} instruction.
The corresponding symbols and color coding of the instructions are shown in Figure~\ref{fig:instruction_set}.

\figDataset

\subsection{Knitting Operation Scheduling}

Given a line of instructions, the sequence of operations is done over a full line using the following steps:\vspace{-2mm}

{\setlength{\leftmargini}{4mm}
\begin{enumerate}
\setlength\itemsep{0mm}
    \item The current stitches are transferred to the new instruction side without racking;
    \item The base operation (\emph{knit}, \emph{tuck} or \emph{miss}) is executed;
    \item The needles of all transfer-related instructions are transferred to the opposite bed without racking;
    \item Instructions that involve moving within a bed proceed to transfer back to the initial side using the appropriate racking and order;
    \item Stack instructions transfer back to the initial side without racking.
\end{enumerate}
}

\paragraph{Instruction Side}
The only instructions requiring an associated bed \emph{side} are those performing a \emph{knit} operation.
We thus encode the bed side in the instructions (\emph{knit}, \emph{purl}, \emph{moves}), except for those where the side can be inferred from the local context.
This inference applies to \emph{Cross} which use the same side as past instructions (for aesthetic reasons), and \emph{Stack} which uses the side of its associated \emph{Move} instruction.
Although this is a simplification of the design space, we did not have any pattern with a different behaviour.

\section{Dataset for Knitting Patterns}

Before developing a learning pipeline, we describe our dataset and its acquisition process.
The frequency of different instruction types is shown in Figure~\ref{fig:dataset}.

The main challenge is that, while machine knitting can produce a large amount of pattern data reasonably quickly, we still need to specify these patterns (and thus generate reasonable pattern instructions), and acquire calibrated images for supervised learning.

\figSupervisedData

\subsection{Pattern Instructions}

We extracted pattern instructions from the proprietary software KnitPaint~\cite{Shima11}. These patterns have various sizes and span a large variety of designs from cable patterns to pointelle stitches, lace, and regular reverse stitches.

Given this set of initial patterns (around a thousand), we normalized the patterns by computing crops of $20\times 20$ instructions with $50\%$ overlap, while using default front stitches for the background of smaller patterns. This provided us with 12,392 individual $20\times 20$ patterns (after pruning invalid patterns since random cropping can destroy the structure).

We then generated the corresponding images in two different ways:
(1) by knitting a subset of 1,044 patches, \ie, Real data, and (2) by rendering all of them using the basic pattern preview from KnitPaint, \ie, Simulated data. 
See Figure~\ref{fig:supervised_data} for sample images.

\subsection{Knitting Many Samples}

The main consideration for capturing knitted patterns is that their tension should be as regular as possible so that knitting units would align with corresponding pattern instructions.
We initially proceeded with knitting and capturing patterns individually but this proved to not be scalable.

We then chose to knit sets of $25$ patterns over a $5\times 5$ tile grid, each of which would be separated by both horizontal and vertical tubular knit structures.
The tubular structures are designed to allow sliding $1/8$ inch steel rods which we use to normalize the tension, as shown in Figure~\ref{fig:acquisition}.
Note that each knitted pattern effectively provides us with two full opposite patterns (the front side, and its back whose instructions can be directly mapped from the front ones). This doubles the size of our real knitted dataset to \emph{2,088 samples} after annotating and cropping the knitted samples.

\figAcquisition[t]

\section{Instruction Synthesis Model}

We present our deep neural network model that infers a 2D knitting instruction map from an image of patterns.
In this section, we provide the theoretical motivation of our framework, and then we describe the loss functions we used, as well as implementation details. 

\subsection{Learning from different domains}
When we have a limited number of real data, it is appealing to leverage simulated data because high quality annotations are automatically available. 
However, learning from synthetic data is problematic due to apparent domain gaps between synthetic and real data.
We study how we can further leverage simulated data.
We are motivated by the recent {work}, Simulated+Unsupervised (S+U) learning~\cite{shrivastava2017learning}, but in contrast to them, we develop our framework from the generalization error perspective.

Let $\calX$ be input space (image), and $\calY$ output space (instruction label), and $\calD$ a data distribution on $\calX$ paired with
a true labeling function $y_\calD{:}\,\calX {\rightarrow} \calY$.
As a typical learning problem, we seek a hypothesis classifier $h{:}\,\calX {\rightarrow} \calY$ that best fits the target function $y$ in terms of an expected loss: $\calL_\calD(h,h') {=} \E_{x\sim \calD}[l\left( h(x), h'(x) \right)]$ for classifiers $h,h'$, where $l{:}\,\calY{\times}\calY \,{\rightarrow}\, \Real_+$ denotes a loss function.
We denote its empirical loss as ${\hat \calL_{\hat\calD}}(h,h') = \tfrac{1}{|{\hat\calD}|} \sum_{i=1}^{|{\hat\calD}|} l(h(x_i), h'(x_i))$, where $\hat\calD {=} \{x\}$ is the sampled dataset.

In our problem, since we have two types of data available, a source domain $\calD_S$ and a target domain $\calD_T$ (which is real or simulated as specified later), our goal is to find $h$ by minimizing the combination of empirical source and target losses as $\alpha$-mixed loss, 
${\hat\calL}_\alpha(h, y) = \alpha {\hat\calL}_{S}(h, y) + (1{-}\alpha) {\hat\calL}_{T}(h, y),$ where $0 {\leq} \alpha {\leq} 1$, and for simplicity we shorten $\calL_{\calD_{\{S,T\}}} {=} \calL_{\{S,T\}}$ and 
we use the parallel notation $\calL_{\{S,T\}}$ and $\hat{\calL}_{\{S,T\}}$.
Our underlying goal is to achieve a minimal generalized target loss $\calL_{T}$.
To develop a generalizable framework, we present a bound over the target loss in terms of its empirical $\alpha$-mixed loss, which is a slight modification of Theorem 3 of \cite{ben2010theory}.

\begin{theorem}
\label{thm:generalization_bound}
Let $\calH$ be a hypothesis class, and
$\calS$ be a labeled sample of
size $m$ generated by drawing $\beta m$ samples from $\calD_S$ and $(1-\beta) m$ samples from $\calD_T$ and labeling
them according to the true label $y$. 
Suppose $\calL$ is symmetric and obeys the triangle inequality.
Let $\hat h \in \calH$ be the empirical minimizer of \hbox{${\hat h} = \argmin_{h} {\hat \calL}_\alpha(h,y)$} on $\calS$ for a fixed $\alpha \in [0,1]$, 
and \hbox{$h^*_T = \argmin_h \calL_T(h,y)$} the target error minimizer.
Then, for any $\delta \in (0,1)$, with
probability at least $1 - \delta$ (over the choice of the samples), we have
\begin{equation}
\label{eq:generalization_bound}
\tfrac{1}{2}|\calL_T({\hat h},y) - {\calL}_T(h^*_T,y)| \leq  \alpha \left( \mathrm{disc}_\calH(\calD_S,\calD_T) + \lambda \right) + \epsilon ,
\end{equation}
where $\epsilon(m, \alpha, \beta, \delta) = \sqrt{ \tfrac{1}{2m} \left( \tfrac{\alpha^2}{\beta} + \tfrac{(1-\alpha)^2}{1-\beta} \right) \log(\tfrac{2}{\delta}) }$, and \hbox{$\lambda {=} \min_{h\in\calH} \calL_{S}(h, y) {+} \calL_{T}(h,y)$}.
\end{theorem}

The proof can be found in the supplementary material. 
Compared to \cite{ben2010theory}, Theorem~\ref{thm:generalization_bound} is purposely extended to use a more general definition of discrepancy $\textrm{disc}_{\calH}(\cdot, \cdot)$ \cite{mansour2009domain} that measures the discrepancy of two distributions (the definition can be found in the supplementary material) and to be agnostic to the model type (simplification), so that we can clearly present our motivation of our model design.

Theorem~\ref{thm:generalization_bound} shows that mixing two sources of data is possible to achieve a better generalization in the target domain.
The bound is always at least as tight as either of $\alpha=0$ or $\alpha=1$ (The case that uses either source or target dataset alone).
Also, as the total number of the combined data sample $m$ is larger, a tighter bound can be obtained.
We consider 0-1 loss for $l$ in this section for simplicity, but not limited to.

A factor that the generalization gap (the right hand side in \Eref{thm:generalization_bound}) strongly depends on is the discrepancy $\mathrm{disc}_\calH(\calD_S,\calD_T)$. 
This suggests that we can achieve a tighter bound if we can reduce $\mathrm{disc}_\calH(\calD_S,\calD_T)$.
We re-parameterize the target distribution $\calD_T$ as $\calD_R$ so that $\calD_T {=} g\circ \calD_R$, where $g$ is a distribution mapping function. 
Then, we find the mapping $g^*$ that leads to the minimal discrepancy for the empirical distribution ${\hat \calD}_R$ as:\vspace{-1mm}
\begin{align}
    g^* =& \argmin\nolimits_g \mathrm{disc}_\calH({\hat\calD}_S,~g\circ {\hat\calD}_R) \nonumber\\
       =& \argmin\nolimits_g \max_{h,h'\in \calH} | \calL_{{\hat\calD}_S}(h,h') - \calL_{g\circ {\hat\calD}_R}(h,h') |,
\end{align}
which is a min-max problem.
Even though the problem is defined for an empirical distribution, it is intractable to search the entire solution space; 
thus, motivated by \cite{ganin2016domain}, we approximately minimize the discrepancy by generative adversarial networks (GAN)~\cite{goodfellow2014generative}.
Therefore, deriving from Theorem~\ref{thm:generalization_bound}, our empirical minimization is formulated by minimizing the convex combination of source and target domain losses as well as the discrepancy as:
\begin{equation}
{\hat h, \hat g} = \argmin_{h\in \calH, g \in \calG} {\hat \calL}_\alpha(h,y) + \tau\cdot\mathrm{disc}_\calH({\hat\calD}_S, g{\circ} {\hat\calD}_R).
\end{equation}

Along with leveraging GAN, our key idea for reducing the discrepancy between two data distributions, \ie, domain gap, is to transfer the real knitting images (target domain, ${\hat\calD}_R$) to synthetic looking data (source domain, ${\hat\calD}_S$) rather than the other way around, \ie, making ${\hat\calD}_S \approx \hat g \circ{\hat\calD}_R$. 
The previous methods have investigated generating realistic looking images to adapt the domain gap.
However, we observe that, when simulated data is mapped to real data, the mapping is a one-to-many mapping due to real-world effects, such as lighting variation, geometric deformation, background clutter, noise, \etc.
This introduces an unnecessary challenge to learn $g(\cdot)$; thus, we instead learn to neutralize the real-world perturbation by mapping from real data to synthetic looking data.
Beyond simplifying the learning of $g(\cdot)$, it also allows the mapping to be utilized at test time for processing of real-world images.

We implement $h$ and $g$ using convolutional neural networks (CNN), and formulate the problem as a local instruction classification\footnote{While our program synthesis can be regarded as a multi-class classification, for simplicity, we consider the simplest binary classification here. However, multi-class classification can be extended by a combination of binary classifications
 \cite{shalev2014understanding}.} 
 and represent the output as a 2D array of classification vectors $\Vec{s}_{(i,j)} \in [0;1]^K$ (\ie, softmax values over $k\in K$) for our $K = 17$ instructions at each spatial location $(i,j)$.
In the following, we describe the loss we use to train our model $h\circ g$ and details about our end-to-end training procedure.

\newcommand{\sumi}[1][h]{\sum_{i=1}^{#1}}
\newcommand{\sumj}[1][w]{\sum_{j=1}^{#1}}
\newcommand{\sumk}{\sum_{k=1}^{K}}
\newcommand{\sumd}{\sum_{d=1}^8}
\newcommand{\ifactor}[1]{\dfrac{1}{#1}}

\paragraph{Loss function}
We use the cross entropy for the loss $\calL$.
We supervise the inferred instruction to match the ground-truth instruction using the standard multi-class cross-entropy
\hbox{$\textrm{CE}(\Vec{s}, \Vec{y}) = - \sum_k y_k\log\left( {s_k}\right)$}
where $s_k$ is the predicted likelihood (softmax value) for instruction $k$, which we compute at each spatial location $(i,j)$.

For synthetic data, we have precise localization of the predicted instructions.
In the case of the real knitted data, human annotations are imperfect and this can cause a minor spatial misalignment of the image with respect to the original instructions. For this reason, we allow the predicted instruction map to be globally shifted by up to one instruction. 
In practice, motivated by multiple instance learning~\cite{dietterich1997solving}, we consider the minimum of the per-image cross-entropy over all possibles one-pixel shifts (as well as the default no-shift variant), \ie, our complete cross entropy loss is\vspace{-1mm}
\begin{equation}
\label{eq:mil_xentropy}
    \mathcal{L}_\textrm{CE} = \frac{1}{Z_{CE}} \min_{d} \sum\nolimits_{i,j \in \mathcal{N}_s} \textrm{CE}(\Vec{s}_{(i,j) + d}, \Vec{y}_{(i,j)}),
\end{equation}
where $d\in \{(dx,dy)\,|\,dx,dy\in\{-1,0,+1\}\}$ is the pattern displacement for the real data and $d \in \{(0,0)\}$ for the synthetic data.
The loss is accumulated over the spatial domain $\mathcal{N}_s = \{2,\dots,w{-}1\} {\times} \{2,\dots,h{-}1\}$ for the instruction map size $w \times h$ reduced by boundary pixels. $Z_{CE} = |\mathcal{N}_s|$ is a normalization factor.

\subsection{Implementation details}
{Our base architecture is illustrated in Figure~\ref{fig:teaser}.}
We implemented it using TensorFlow~\cite{abadi2016tensorflow}.
The prediction network \texttt{Img2prog} takes $160{\times} 160$ grayscale images as input and generates $20{\times}20$ instruction maps.
The structure consists of an initial set of 3 convolution layers with stride 2 that downsample the image to $20{\times}20$ spatial resolution, a feature transformation part made of $6$ residual blocks~\cite{he2016deep, Zhu17}, and two final convolutions producing the instructions. 
The kernel size of all convolution layers is $3{\times}3$, except for the last layer which is $1{\times}1$.
We use instance normalization~\cite{Ulyanov16} for each of the initial down-convolutions, and ReLU everywhere.

We solve the minimax problem of the discrepancy $\textrm{disc}(\cdot, \cdot)$ \wrt $g$ using the least-square Patch-GAN \cite{Zhu17}.
Additionally, we add the perceptual loss and style loss~\cite{johnson2016perceptual} between input real images and its generated images and between simulated images and generated images, respectively, to regularize the GAN training, which stably speeds up the training of $g$.

The structure of the \texttt{Refiner} network $g$ and the balance between losses can be found in the supplementary.

\paragraph{Training procedure}
We train our network with a combination of the real knitted patterns and the rendered images. 
We have oversampled the real data to achieve 1:1 mix ratio with several data augmentation strategies, which can be found in the supplementary material.
We train with 80\% of the real data, withholding 5\% for validation and 15\% for testing, whereas we use all the synthetic data for training.

According to the typical training method for GAN~\cite{goodfellow2014generative}, we alternate the training between discriminator and the other networks, $h$ and $g$, but we update the discriminator only every other iteration, and the iteration is counted according to the number of updates for $h$ and $g$.

We trained our model for $150k$ iterations with batch size $2$ for each domain data using {ADAM} optimizer with initial learning rate $0.0005$, exponential decay rate $0.3$ every $50,000$ iterations. The training took from $3$ to ${4}$ hours (depending on the model) on a Titan Xp GPU.

 \section{Experiments}

We first evaluate baseline models for our new task, along with an ablation study looking at the impact of our loss and the trade-off between real and synthetic data mixing.
Finally, we look at the impact of the size of our dataset..

\paragraph{Accuracy Metric}
For the same reason our loss in \Eref{eq:mil_xentropy} takes into consideration a 1-pixel ambiguity along the spatial domain, we use a similarly defined accuracy. 
It is measured by the average of 
$\max_{d} \tfrac{1}{N_\textrm{inst}}\sum_{i,j} \mathbb{I}[{y_\textrm{GT}}_{(i,j)}\,{=}\,\argmax_k {s}_{(i,j) + d}^k]$
over the whole dataset,
where $N_\textrm{inst} = Z_\textrm{CE}$ is the same normalization constant as in \Eref{eq:mil_xentropy}, $y_\textrm{GT}$ the ground-truth label,
$\mathbb{I}[\cdot]$ is the indicator function that returns 1 if the statement is true, 0 otherwise.
We report two variants: \texttt{FULL} averages over all the instructions, whereas \texttt{FG} considers all instructions but the background (\ie, discard the most predominant instruction in the pattern).

\paragraph{Perceptual Metrics}
For the baselines and the ablation experiments, we additionally provide perceptual metrics that measure how similar the knitted pattern would look.
An indirect method for evaluation is to apply a pre-trained neural network to generated images and calculate statistics of its
output, \eg, Inception Score~\cite{salimans2016improved}.
Inspired by this, we learn a separate network to render simulated images of
the generated instructions and compare it to the rendering of the ground truth using standard PSNR and SSIM metrics.
Similarly to the accuracy, 
we allow for one instruction shift, which translates to full 8 pixels shifts in the image domain.

\subsection{Comparison to Baselines}

\Tref{tab:comparisons} compares the measured accuracy of predicted instructions on our real image test set.
We also provide qualitative results in \refFig{baselines}.

\tabComparisons

The first 5 rows of \Tref{tab:comparisons}-(a1-5) present results of previous works to provide snippets of other domain methods.
For CycleGAN, no direct supervision is provided and the domains are mapped in a fully unsupervised manner.
Together with Pix2pix, the two first methods do not use cross-entropy but L1 losses with GAN.
Although they can provide interesting image translations, they are not specialized for multi-class classification problems, and thus cannot compete.
All baselines are trained from scratch.
Furthermore, since their architectures use the same spatial resolution for both input and output, we up-sampled instruction maps to the same image dimensions using nearest neighbor interpolation.

S+U Learning \cite{shrivastava2017learning} used a refinement network to generate a training dataset that makes existing synthetic data look realistic.
In this case, our implementation uses our base network \texttt{Img2prog} and approximates real domain transfer by using style transfer.
We tried two variants: using the original Neural Style Transfer~\cite{Gatys16} and CycleGAN~\cite{Zhu17}.
Both input data types lead to very similar accuracy (negligible difference) when added as a source of real data.
We thus only report the numbers from the first one \cite{Gatys16}.

\tabPerInstruction

\subsection{Impact of Loss and Data Mixing Ratio}

The second group in \Tref{tab:comparisons}-(b1-2) considers our base network $h$ (\texttt{Img2prog}) without the refinement network $g$ (\texttt{Refiner}) that translates real images onto the synthetic domain.
In this case, \texttt{Img2prog} maps real images directly onto the instruction domain.
The results generated by all direct image translation networks trained with cross-entropy (a3-5) compare similarly with our base \texttt{Img2prog} on both accuracy and perceptual metrics.
This shows our base network allows a fair comparison with the competing methods, and as will be shown, our final performance (c1-5, d1) is not gained
from the design of \texttt{Img2prog} but \texttt{Refiner}.

The third group in \Tref{tab:comparisons}-(c1-5) looks at the impact of the mixing ratio $\alpha$ when using our full architecture.
In this case, the refinement network $g$ translates our real image into a synthetic looking one, which is then translated by \texttt{Img2prog} into instructions.
This combination favorably improves both the accuracy and perceptual quality of the results with the best mixing ratio of $\alpha {=} 2/3$ as well as a stable performance regime of $\alpha\in[0.5,0.9]$, which favors more the supervision from diverse simulated data.
While $\epsilon$ in Theorem~\ref{thm:generalization_bound} has a minimum at $\alpha{=}\beta$, we have a biased $\alpha$ due to other effects, $\textrm{disc}(\cdot)$ and $\lambda$.

We tried learning the opposite mapping $g$ (from synthetic image to realistic looking), while directly feeding real data to $h$.
This leads to detrimental results with mode collapsing.
The learned $g$ maps to a trivial pattern and texture that injects the pattern information in invisible noise -- \ie, adversarial perturbation -- to enforce that $h$  maintains a plausible inference.
We postulate this might be due to the non-trivial one-to-many mapping relationship from simulated data to real data, and overburden for $h$ to learn to compensate real perturbations by itself.

In the last row of \Tref{tab:comparisons}-(d1), we present the result obtained with a variant network, \texttt{Img2prog++} which additionally uses skip connections from each down-convolution of \texttt{Img2prog} to increase its representation power.
This is our best model in the qualitative comparisons of \refFig{baselines}. 

Finally, we check the per-instruction behavior of our best model, shown through the per-instruction accuracy in Table~\ref{tab:per_instruction}.
Although there is a large difference in instruction frequency, our method still manages to learn some useful representation for rare instructions but the variability is high.
This suggests the need for a systematic way of tackling the class imbalance~\cite{Huang16, Lin18}.

\subsection{Impact of Dataset Size}

\figDatasetAccuracy
In Figure~\ref{fig:dataset_accuracy}, we show the impact of the real data amount on accuracy.
As expected, increasing the amount of training data helps (and we have yet to reach saturation).
With low amounts of data (here $400$ samples or less), the training is not always stable -- some data splits lead to overfitting.

\figBaselines

\section{Discussion and Related Work}

\noindent\textbf{Knitting instruction generation}\quad
We introduce automatic program synthesis for machine kitting using deep images translation.
Recent works allow automatic conversion of 3D meshes to machine instructions~\cite{Narayanan18}, or directly model garment patterns on specialized meshes~\cite{Yuksel2012, Wu18a}, which can then be translated into hand knitting instruction~\cite{Wu18b}.
While this does enable a wide range of achievable patterns, the accompanying interface requires stitch-level specification.
This can be tedious, requires prior knitting experience and the resulting knits are not machine-knittable. 
We bypass the complete need of modeling these patterns and allow direct synthesis from image exemplars that are simpler to acquire and also machine knittable.

\noindent\textbf{Simulated data based learning}\quad
We demonstrate a way to effectively leverage both simulated and real knitting data.
There have been a recent surge of adversarial learning based domain adaptation methods~\cite{shrivastava2017learning,tzeng2017adversarial,hoffman2018cycada} in the simulation-based learning paradigm.
They deploy GANs and refiners to refine the synthetic or simulated data to look real.
We instead take the opposite direction to exploit the simple and regular domain properties of synthetic data.
Also, while they require multi-step training, our networks are end-to-end trained from scratch and only need a one-sided mapping rather than a two-sided cyclic mapping~\cite{hoffman2018cycada}.

\noindent\textbf{Semantic segmentation}\quad
Our problem is to transform photographs of knit structures into their corresponding instruction maps.
This resembles semantic segmentation which is a per-pixel multi-class classification problem except that the spatial extent of individual instruction interactions is much larger when looked at from the original image domain.
From a program synthesis perspective, we have access to a set of constraints on valid instruction interactions (e.g. \emph{Stack} is always paired with a \emph{Move} instruction reaching it).
This conditional dependency is referred to as context in semantic segmentation, and there have been many efforts to explicitly tackle this by Conditional Random Field (CRF) \cite{zheng2015conditional,chen2018deeplab,rother2004grabcut}.
They clean up spurious predictions of a weak classifier by favoring same-label assignments to neighboring pixels, \eg, Potts model.
For our problem, we tried a first-order syntax compatibility loss, but there was no noticeable improvement.
However we note that \cite{YuKoltun2016} observed that a CNN with a large receptive field but without CRF can outperform or compare similarly to its counterpart with CRF for subsequent structured guidance~\cite{zheng2015conditional,chen2018deeplab}.
While we did not consider any CRF post processing in this work, sophisticated modeling of the knittability would be worth exploring as a future direction.

Another apparent difference between knitting and semantic segmentation is that semantic segmentation is an easy -- although tedious -- task for humans, whereas parsing knitting instructions requires vast expertise or reverse engineering.

Finally, we tried to use a state-of-the-art scene parsing network with very large capacity and pretraining~\cite{Zhou16} which led to similar results to our best performing setup, but with a significantly more complicated model and training time.
See the supplementary for details.

\noindent\textbf{Neural program synthesis}\quad
In terms of returning explicit interpretable programs, our work is closely related to program synthesis, which is a \ignore{traditional} challenging, ongoing problem.\footnote{A similar concept is \emph{program induction}, in which the model learns to mimic the program rather than explicitly return it. From our perspective, semantic segmentation is closer to program induction, while our task is program synthesis.}
The recent advance of deep learning has made notable progress in this domain, \eg, \cite{johnson2017inferring, devlin2017robustfill}.
Our task would have potentials to extend the research boundary of this field, since it differs from any other prior task on program synthesis in that:
1) while program synthesis solutions adopt a sequence generation paradigm~\cite{kant2018recent},
our type of input-output pairs are 2D program maps, and
2) the domain specific language \ignore{(our instruction set)} is newly developed and
applicable to practical knitting.

\noindent\textbf{Limitations}\quad
This work has two main limitations:
(1) it does not explicitly model the pattern scale; and
(2) it does not impose hard constraints on the output semantics, thus the intent of some instructions may be violated.
We provide preliminary scale identification results in the supplementary, together with the details on the necessary post-processing that enables machine knitting of any output.
\section{Conclusion}

We have proposed an inverse process for translating high level specifications to manufacturing instructions based on deep learning.
In particular, we have developed a framework that translates images of knitted patterns to instructions for industrial whole-garment knitting machines.
In order to realize this framework, we have collected a dataset of machine instructions and corresponding images of knitted patterns.
We have shown both theoretically and empirically how we can improve the quality of our translation process by combining synthetic and real image data.
We have shown an uncommon usage of synthetic data to develop a model that maps real images onto a more regular domain from which machine instructions can more easily be inferred.

The different trends between our perceptual and semantic metrics bring the question of whether adding a perceptual loss on the instructions might also help improve the semantic accuracy.
This could be done with a differentiable rendering system.
Another interesting question is whether using higher-accuracy simulations~\cite{Yuksel2012, Wu18a} could help and how the difference in regularity affects the generalization capabilities of our prediction.

We believe that our work will stimulate research to develop machine learning methods for design and manufacturing.


\section*{Acknowledgements}

We would like to thank Jim McCann and his colleagues at the Carnegie Mellon Textiles Lab for providing us with the necessary tools to programmatically write patterns for our industrial knitting machine.

\bibliographystyle{icml2019}
\bibliography{paper.bib}

\begin{thebibliography}{42}
\providecommand{\natexlab}[1]{#1}
\providecommand{\url}[1]{\texttt{#1}}
\expandafter\ifx\csname urlstyle\endcsname\relax
  \providecommand{\doi}[1]{doi: #1}\else
  \providecommand{\doi}{doi: \begingroup \urlstyle{rm}\Url}\fi

\bibitem[{Abadi et al.}(2016)]{abadi2016tensorflow}
{Abadi et al.}
\newblock Tensorflow: a system for large-scale machine learning.
\newblock In \emph{OSDI}, 2016.

\bibitem[Ben-David et~al.(2010)Ben-David, Blitzer, Crammer, Kulesza, Pereira,
  and Vaughan]{ben2010theory}
Ben-David, S., Blitzer, J., Crammer, K., Kulesza, A., Pereira, F., and Vaughan,
  J.~W.
\newblock A theory of learning from different domains.
\newblock \emph{Machine learning}, 79\penalty0 (1-2):\penalty0 151--175, 2010.

\bibitem[Chen et~al.(2018)Chen, Papandreou, Kokkinos, Murphy, and
  Yuille]{chen2018deeplab}
Chen, L.-C., Papandreou, G., Kokkinos, I., Murphy, K., and Yuille, A.~L.
\newblock Deeplab: Semantic image segmentation with deep convolutional nets,
  atrous convolution, and fully connected crfs.
\newblock \emph{{IEEE} Transactions on Pattern Analysis and Machine
  Intelligence}, 40\penalty0 (4):\penalty0 834--848, 2018.

\bibitem[Crammer et~al.(2008)Crammer, Kearns, and Wortman]{crammer2008learning}
Crammer, K., Kearns, M., and Wortman, J.
\newblock Learning from multiple sources.
\newblock \emph{Journal of Machine Learning Research}, 9\penalty0
  (Aug):\penalty0 1757--1774, 2008.

\bibitem[Devlin et~al.(2017)Devlin, Uesato, Bhupatiraju, Singh, Mohamed, and
  Kohli]{devlin2017robustfill}
Devlin, J., Uesato, J., Bhupatiraju, S., Singh, R., Mohamed, A.-r., and Kohli,
  P.
\newblock Robustfill: Neural program learning under noisy i/o.
\newblock In \emph{International Conference on Machine Learning}, 2017.

\bibitem[Dietterich et~al.(1997)Dietterich, Lathrop, and
  Lozano-P{\'e}rez]{dietterich1997solving}
Dietterich, T.~G., Lathrop, R.~H., and Lozano-P{\'e}rez, T.
\newblock Solving the multiple instance problem with axis-parallel rectangles.
\newblock \emph{Artificial intelligence}, 89\penalty0 (1-2):\penalty0 31--71,
  1997.

\bibitem[Donohue(2015)]{Donohue15}
Donohue, N.
\newblock \emph{750 Knitting Stitches: The Ultimate Knit Stitch Bible}.
\newblock St. Martin's Griffin, 2015.

\bibitem[Galanti \& Wolf(2017)Galanti and Wolf]{galanti2017theory}
Galanti, T. and Wolf, L.
\newblock A theory of output-side unsupervised domain adaptation.
\newblock \emph{arXiv:1703.01606}, 2017.

\bibitem[Ganin et~al.(2016)Ganin, Ustinova, Ajakan, Germain, Larochelle,
  Laviolette, Marchand, and Lempitsky]{ganin2016domain}
Ganin, Y., Ustinova, E., Ajakan, H., Germain, P., Larochelle, H., Laviolette,
  F., Marchand, M., and Lempitsky, V.
\newblock Domain-adversarial training of neural networks.
\newblock \emph{Journal of Machine Learning Research}, 17\penalty0
  (1):\penalty0 2096--2030, 2016.

\bibitem[Gatys et~al.(2016)Gatys, Ecker, and Bethge]{Gatys16}
Gatys, L.~A., Ecker, A.~S., and Bethge, M.
\newblock Image style transfer using convolutional neural networks.
\newblock In \emph{{IEEE} Conference on Computer Vision and Pattern
  Recognition}, 2016.

\bibitem[Goodfellow et~al.(2014)Goodfellow, Pouget-Abadie, Mirza, Xu,
  Warde-Farley, Ozair, Courville, and Bengio]{goodfellow2014generative}
Goodfellow, I., Pouget-Abadie, J., Mirza, M., Xu, B., Warde-Farley, D., Ozair,
  S., Courville, A., and Bengio, Y.
\newblock Generative adversarial nets.
\newblock In \emph{Advances in Neural Information Processing Systems}, 2014.

\bibitem[Guo et~al.(2017)Guo, Pleiss, Sun, and Weinberger]{guo2017calibration}
Guo, C., Pleiss, G., Sun, Y., and Weinberger, K.~Q.
\newblock On calibration of modern neural networks.
\newblock In \emph{Proceedings of the 34th International Conference on Machine
  Learning-Volume 70}, pp.\  1321--1330. JMLR. org, 2017.

\bibitem[He et~al.(2016)He, Zhang, Ren, and Sun]{he2016deep}
He, K., Zhang, X., Ren, S., and Sun, J.
\newblock Deep residual learning for image recognition.
\newblock In \emph{{IEEE} Conference on Computer Vision and Pattern
  Recognition}, 2016.

\bibitem[Hoffman et~al.(2018)Hoffman, Tzeng, Park, Zhu, Isola, Saenko, Efros,
  and Darrell]{hoffman2018cycada}
Hoffman, J., Tzeng, E., Park, T., Zhu, J.-Y., Isola, P., Saenko, K., Efros,
  A.~A., and Darrell, T.
\newblock Cycada: Cycle-consistent adversarial domain adaptation.
\newblock In \emph{International Conference on Machine Learning}, 2018.

\bibitem[Huang et~al.(2016)Huang, Li, Loy, and Tang]{Huang16}
Huang, C., Li, Y., Loy, C.~C., and Tang, X.
\newblock Learning deep representation for imbalanced classification.
\newblock In \emph{IEEE Conference on Computer Vision and Pattern Recognition
  (CVPR)}, 2016.

\bibitem[Isola et~al.(2017)Isola, Zhu, Zhou, and Efros]{Isola16}
Isola, P., Zhu, J.-Y., Zhou, T., and Efros, A.~A.
\newblock Image-to-image translation with conditional adversarial networks.
\newblock In \emph{{IEEE} Conference on Computer Vision and Pattern
  Recognition}, 2017.

\bibitem[Johnson et~al.(2016)Johnson, Alahi, and
  Fei-Fei]{johnson2016perceptual}
Johnson, J., Alahi, A., and Fei-Fei, L.
\newblock Perceptual losses for real-time style transfer and super-resolution.
\newblock In \emph{European Conference on Computer Vision}, 2016.

\bibitem[Johnson et~al.(2017)Johnson, Hariharan, van~der Maaten, Hoffman,
  Fei-Fei, Zitnick, and Girshick]{johnson2017inferring}
Johnson, J., Hariharan, B., van~der Maaten, L., Hoffman, J., Fei-Fei, L.,
  Zitnick, C.~L., and Girshick, R.
\newblock Inferring and executing programs for visual reasoning.
\newblock In \emph{{IEEE} International Conference on Computer Vision}, 2017.

\bibitem[Kant(2018)]{kant2018recent}
Kant, N.
\newblock Recent advances in neural program synthesis.
\newblock \emph{arXiv:1802.02353}, 2018.

\bibitem[Lin et~al.(2018)Lin, Narayanan, and McCann]{Lin18}
Lin, J., Narayanan, V., and McCann, J.
\newblock Efficient transfer planning for flat knitting.
\newblock In \emph{Proceedings of the 2nd ACM Symposium on Computational
  Fabrication}, pp.\ ~1. ACM, 2018.

\bibitem[Mansour et~al.(2009)Mansour, Mohri, and
  Rostamizadeh]{mansour2009domain}
Mansour, Y., Mohri, M., and Rostamizadeh, A.
\newblock Domain adaptation: Learning bounds and algorithms.
\newblock In \emph{Conference on Learning Theory}, 2009.

\bibitem[McCann et~al.(2016)McCann, Albaugh, Narayanan, Grow, Matusik, Mankoff,
  and Hodgins]{McCann16}
McCann, J., Albaugh, L., Narayanan, V., Grow, A., Matusik, W., Mankoff, J., and
  Hodgins, J.
\newblock A compiler for 3d machine knitting.
\newblock \emph{ACM Transactions on Graphics}, 35\penalty0 (4):\penalty0 49,
  2016.

\bibitem[Narayanan et~al.(2018)Narayanan, Albaugh, Hodgins, Coros, and
  McCann]{Narayanan18}
Narayanan, V., Albaugh, L., Hodgins, J., Coros, S., and McCann, J.
\newblock Automatic knitting of 3d meshes.
\newblock \emph{ACM Transactions on Graphics}, 2018.

\bibitem[Ronneberger et~al.(2015)Ronneberger, Fischer, and Brox]{Ronneberger15}
Ronneberger, O., Fischer, P., and Brox, T.
\newblock U-net: Convolutional networks for biomedical image segmentation.
\newblock In \emph{International Conference on Medical image computing and
  computer-assisted intervention}. Springer, 2015.

\bibitem[Rother et~al.(2004)Rother, Kolmogorov, and Blake]{rother2004grabcut}
Rother, C., Kolmogorov, V., and Blake, A.
\newblock Grabcut: Interactive foreground extraction using iterated graph cuts.
\newblock \emph{ACM Transactions on Graphics}, 23\penalty0 (3):\penalty0
  309--314, 2004.

\bibitem[Russakovsky et~al.(2015)Russakovsky, Deng, Su, Krause, Satheesh, Ma,
  Huang, Karpathy, Khosla, Bernstein, et~al.]{russakovsky2015imagenet}
Russakovsky, O., Deng, J., Su, H., Krause, J., Satheesh, S., Ma, S., Huang, Z.,
  Karpathy, A., Khosla, A., Bernstein, M., et~al.
\newblock Imagenet large scale visual recognition challenge.
\newblock \emph{International journal of computer vision}, 115\penalty0
  (3):\penalty0 211--252, 2015.

\bibitem[Salimans et~al.(2016)Salimans, Goodfellow, Zaremba, Cheung, Radford,
  and Chen]{salimans2016improved}
Salimans, T., Goodfellow, I., Zaremba, W., Cheung, V., Radford, A., and Chen,
  X.
\newblock Improved techniques for training gans.
\newblock In \emph{Advances in Neural Information Processing Systems}, 2016.

\bibitem[Shalev-Shwartz \& Ben-David(2014)Shalev-Shwartz and
  Ben-David]{shalev2014understanding}
Shalev-Shwartz, S. and Ben-David, S.
\newblock \emph{Understanding machine learning: From theory to algorithms}.
\newblock Cambridge university press, 2014.

\bibitem[Shida \& Roehm(2017)Shida and Roehm]{Shida17}
Shida, H. and Roehm, G.
\newblock \emph{Japanese Knitting Stitch Bible: 260 Exquisite Patterns by
  Hitomi Shida}.
\newblock Tuttle Publishing, 2017.

\bibitem[{Shima Seiki}()]{Shima11}
{Shima Seiki}.
\newblock {SDS-ONE Apex3}.
\newblock {http://www.shimaseiki.com/product/design/sdsone\_apex /flat/}.
\newblock [Online; Accessed: 2018-09-01].

\bibitem[Shrivastava et~al.(2017)Shrivastava, Pfister, Tuzel, Susskind, Wang,
  and Webb]{shrivastava2017learning}
Shrivastava, A., Pfister, T., Tuzel, O., Susskind, J., Wang, W., and Webb, R.
\newblock Learning from simulated and unsupervised images through adversarial
  training.
\newblock In \emph{{IEEE} Conference on Computer Vision and Pattern
  Recognition}, 2017.

\bibitem[Simonyan \& Zisserman(2014)Simonyan and Zisserman]{Simonyan14}
Simonyan, K. and Zisserman, A.
\newblock Very deep convolutional networks for large-scale image recognition.
\newblock \emph{arXiv preprint arXiv:1409.1556}, 2014.

\bibitem[Tzeng et~al.(2017)Tzeng, Hoffman, Saenko, and
  Darrell]{tzeng2017adversarial}
Tzeng, E., Hoffman, J., Saenko, K., and Darrell, T.
\newblock Adversarial discriminative domain adaptation.
\newblock In \emph{{IEEE} Conference on Computer Vision and Pattern
  Recognition}, 2017.

\bibitem[Ulyanov et~al.(2016)Ulyanov, Vedaldi, and Lempitsky]{Ulyanov16}
Ulyanov, D., Vedaldi, A., and Lempitsky, V.
\newblock Instance normalization: The missing ingredient for fast stylization.
\newblock \emph{arXiv:1607.08022}, 2016.

\bibitem[Underwood(2009)]{Underwood09}
Underwood, J.
\newblock The design of 3d shape knitted preforms.
\newblock \emph{Thesis, RMIT University}, 2009.

\bibitem[Wu et~al.(2018{\natexlab{a}})Wu, Gao, Ferguson, Panozzo, and
  Yuksel]{Wu18a}
Wu, K., Gao, X., Ferguson, Z., Panozzo, D., and Yuksel, C.
\newblock Stitch meshing.
\newblock \emph{ACM Transactions on Graphics (SIGGRAPH)}, 37\penalty0
  (4):\penalty0 130:1--130:14, 2018{\natexlab{a}}.

\bibitem[Wu et~al.(2018{\natexlab{b}})Wu, Swan, and Yuksel]{Wu18b}
Wu, K., Swan, H., and Yuksel, C.
\newblock Knittable stitch meshes.
\newblock \emph{ACM Transactions on Graphics}, 2018{\natexlab{b}}.

\bibitem[Yu \& Koltun(2016)Yu and Koltun]{YuKoltun2016}
Yu, F. and Koltun, V.
\newblock Multi-scale context aggregation by dilated convolutions.
\newblock In \emph{International Conference on Learning Representations}, 2016.

\bibitem[Yuksel et~al.(2012)Yuksel, Kaldor, James, and Marschner]{Yuksel2012}
Yuksel, C., Kaldor, J.~M., James, D.~L., and Marschner, S.
\newblock Stitch meshes for modeling knitted clothing with yarn-level detail.
\newblock \emph{ACM Transactions on Graphics (SIGGRAPH)}, 31\penalty0
  (3):\penalty0 37:1--37:12, 2012.

\bibitem[Zheng et~al.(2015)Zheng, Jayasumana, Romera-Paredes, Vineet, Su, Du,
  Huang, and Torr]{zheng2015conditional}
Zheng, S., Jayasumana, S., Romera-Paredes, B., Vineet, V., Su, Z., Du, D.,
  Huang, C., and Torr, P.~H.
\newblock Conditional random fields as recurrent neural networks.
\newblock In \emph{{IEEE} International Conference on Computer Vision}, 2015.

\bibitem[Zhou et~al.(2018)Zhou, Zhao, Puig, Xiao, Fidler, Barriuso, and
  Torralba]{Zhou16}
Zhou, B., Zhao, H., Puig, X., Xiao, T., Fidler, S., Barriuso, A., and Torralba,
  A.
\newblock Semantic understanding of scenes through the ade20k dataset.
\newblock \emph{International Journal of Computer Vision}, 2018.

\bibitem[Zhu et~al.(2017)Zhu, Park, Isola, and Efros]{Zhu17}
Zhu, J.-Y., Park, T., Isola, P., and Efros, A.~A.
\newblock Unpaired image-to-image translation using cycle-consistent
  adversarial networks.
\newblock In \emph{IEEE International Conference on Computer Vision}, 2017.

\end{thebibliography}

\cleardoublepage

\begin{center}
    \Large \bf -- Supplementary Material --\\
    Neural Inverse Knitting: From Images to Manufacturing Instructions
\end{center}

\section*{Contents}
\begin{itemize}
\setlength{\itemsep}{0mm}
    \item[*] Details of the \texttt{Refiner} network.
    \item[*] Loss balancing parameters.
    \item[*] Used data augmentation details.
    \item[*] Pattern scale identification.
    \item[*] Post-processing details.
    \item[*] Additional quantitative results.
    \item[*] Additional qualitative results.
    \item[*] Lemmas and theorem with the proofs.
\end{itemize}

\section*{Additional Resources}
Additional knitting-related resources (the dataset, code and overview videos of the machine knitting process) can be found on our project page:\\ \url{http://deepknitting.csail.mit.edu/}

\section*{The \texttt{Refiner} Network}
Our refinement network translates real images into regular images that look similar to synthetic images. Its implementation is similar to \texttt{Img2prog}, except that it outputs the same resolution image as input, of which illustration is shown in Figure~\ref{fig:network_refiner}.

\section*{Loss Balancing Parameters}
When learning our full architecture with both \texttt{Refiner} and \texttt{Img2prog}, we have three different losses: the cross-entropy loss $\mathcal{L}_{CE}$, the perceptual loss $\mathcal{L}_\textrm{Perc}$, and the PatchGAN loss.

Our combined loss is the weighted sum
\begin{equation}
    \mathcal{L} = \lambda_\textrm{CE} \mathcal{L}_\textrm{CE}
                + \lambda_\textrm{Perc} \mathcal{L}_\textrm{Perc}
                + \lambda_\textrm{GAN} \mathcal{L}_\textrm{GAN}
\end{equation}
where we used the weights:
$\lambda_\textrm{CE} = 3$, $\lambda_\textrm{Perc} = 0.02/(128)^2$ and $\lambda_\textrm{GAN} = 0.2$.
The losses $\mathcal{L}_\textrm{Perc}$ and $\lambda_\textrm{GAN}$ are measured on the output of \texttt{Refiner}, while the loss $\lambda_\textrm{CE}$ is measured on \texttt{Img2prog}.

The perceptual loss~\citep{johnson2016perceptual} consists of the feature matching loss and style loss (using the gram matrix).
If not mentioned here, we follow the implementation details 
of \cite{johnson2016perceptual}, where VGG-16~\cite{Simonyan14} is used for feature extraction, after replacing max-pooling operations with average-pooling.
The feature matching part is done using the \texttt{pool3} layer, comparing the input real image and the output of \texttt{Refiner} so as to preserve the content of the input data.
For the style matching part, we use the gram matrices of the \{\texttt{conv1\_2}, \texttt{conv2\_2}, \texttt{conv3\_3}\} layers with the respective relative weights \{0.3, 0.5, 1.0\}.
The measured style loss is between the synthetic image and the output of \texttt{Refiner}.

For $\calL_\textrm{GAN}$ and the loss for the discriminator, the least-square Patch-GAN loss~\cite{Zhu17} is used.
We used $\{-1,1\}$ for the regression labels for respective fake and real samples insted of the label $\{0,1\}$ used in \cite{Zhu17}.

For training, we normalize the loss $\lambda_\textrm{CE}$ to be balanced according to the data ratio of a batch.
Specifically, for example, suppose a batch consisting of 2 real and 4 synthetic samples, respectively.
Then, we inversely weighted the respective cross entropy losses for real and synthetic data by the weights of $\tfrac{4}{6}$ and $\tfrac{2}{6}$, so that the effects from the losses are balanced.
This encourages the best performance to be expected at near $\alpha=0.5$ within a batch.

\begin{figure}[!t]
    \centering
    \includegraphics[width=0.8\linewidth]{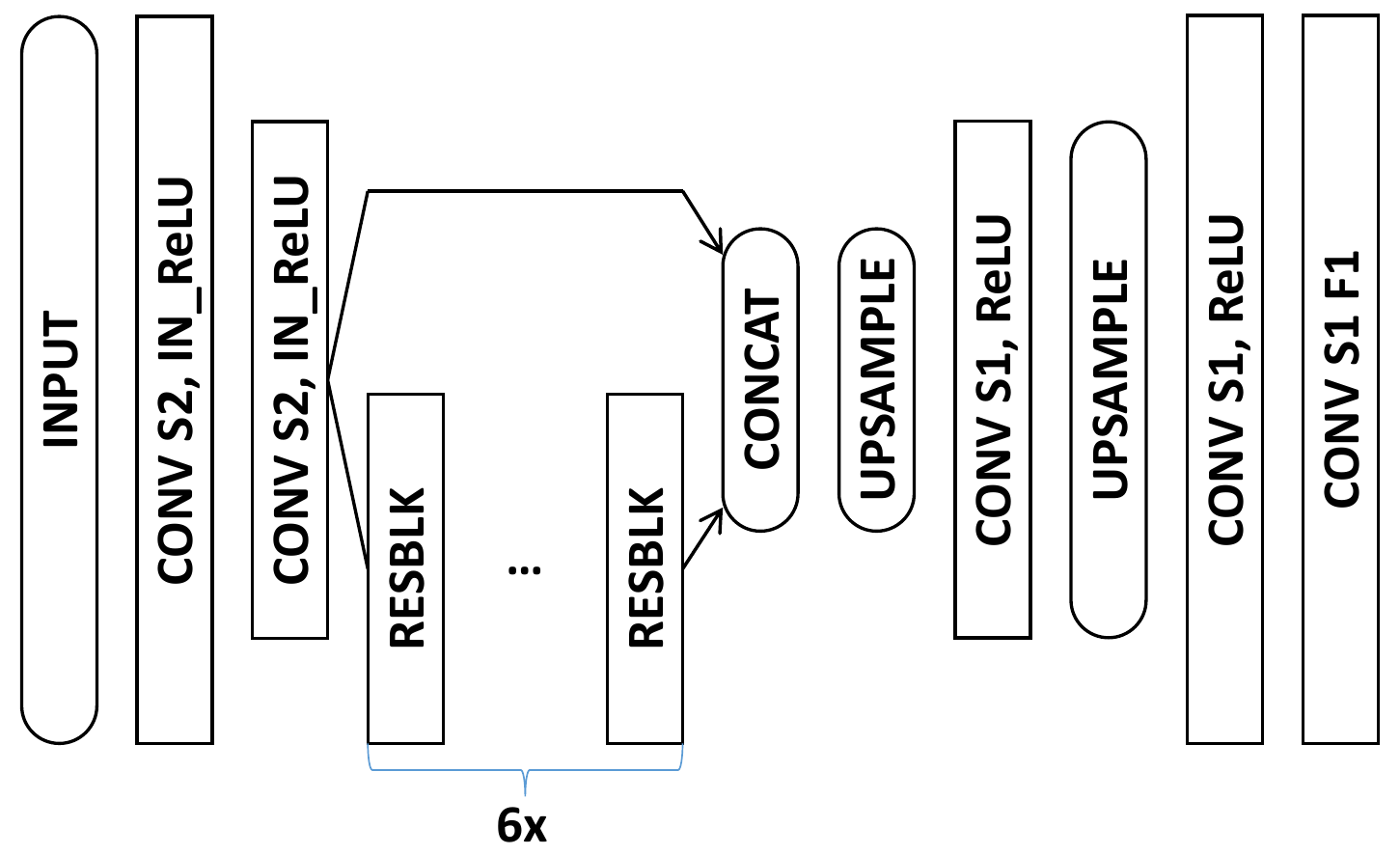}
    \caption{The illustration of the \texttt{Refiner} network architecture, where S$\#N$ denotes the stride size of $\#N$, \texttt{IN\_ReLU} indicates the Instance normalization followed by ReLU, \texttt{Resblk} is the residual block that consists of ConvS1-ReLU-ConvS1 with shortcut connection~\cite{he2016deep}, Upsample is the nearest neighbor upsampling with the factor $2\times$, $F$ is the output channel dimension.
    If not mentioned, the default parameters for all the convolutions are the stride size of 2, $F=64$, and the $3\times 3$ kernel size.
    }
    \label{fig:network_refiner}
\end{figure}

\figScaleEstimationExp

\section*{Data Augmentation}
We use multiple types of data augmentation to notably increase the diversity of yarn colors, lighting conditions, yarn tension, and scale:
\begin{itemize}
    \item \textbf{Global Crop Perturbation}: we add random noise to the location of the crop borders for the real data images, and crop on-the-fly during training; the noise intensity is chosen such that each border can shift at most by half of one stitch;
    \item \textbf{Local Warping}: we randomly warp the input images locally using non-linear warping with linear RBF kernels on a sparse grid.
    We use one kernel per instruction and the shift noise is a 0-centered gaussian with $\sigma$ being $1/5$ of the default instruction extent in image space (i.e. $\sigma = 8/5$);
    \item \textbf{Intensity augmentation}: we randomly pick a single color channel and use it as a mono-channel input, so that it provides diverse  spectral characteristics. Also note that, in order to enhance the intensity scale invariance, we apply instance normalization \cite{Ulyanov16} for the upfront convolution layers of our encoder network.
\end{itemize}

\section*{Pattern scale identification}
Our base system assumes that the input image is taken at a specific zoom level designed for our dataset, which is likely not going to be true for a random image.
We currently assume this to be solved by the user given proper visual feedback (i.e., the user would see the pattern in real-time as they scan their pattern of interest with a mobile phone).

Here, we investigate the potential of automatically discovering the scale of the pattern.
Our base idea is to evaluate the confidence of the output instruction map for different candidate scales and to choose the one with highest confidence.
Although the softmax output cannot directly be considered as a valid probability distribution, it can serve as an approximation, which can be calibrated for~\cite{guo2017calibration}.
As a proof of concept, we take a full $5\times 5$ pattern image from our dataset and crop its center at different scales from $160$ pixels to $2000$ pixels of width.
We then measure the output of the network and compute a \emph{scale pseudo-confidence} as the average over pixels of the maximum softmax component.
 
In Figure~\ref{fig:scale}, we show a sample image with crops at various scales, together with the corresponding uncalibrated pseudo-confidence measure, which peaks at around $600$ pixels scale.
Coincidentally, this corresponds to the scale of our ground truth crops for that image.

This suggests two potential scenarios:
(1) the user takes a much larger image and then that pattern image gets analysed offline to figure out the correct scale to work at using a similar procedure, and then generates a full output by using a tiling of crops at the detected scale, or
(2) an interactive system could provide scale information and suggest the user to get closer to (or farther from) the target depending on the confidence gradient.

\tableLargeNetwork

\section*{Data post-processing}

As mentioned in the main paper, our framework does not enforce hard constraint on the output semantics.
This implies that some outputs may not be machine-knittable as-is.

\newcommand{\cross}{\textsc{Cross}}
\newcommand{\move}{\textsc{Move}}
\newcommand{\stack}{\textsc{Stack}}
\newcommand{\knit}{\textsc{Knit}}

More precisely, the output of our network may contain invalid instructions pairs or a lack thereof.
We remedy to these conflicts by relaxing the conflicting instruction, which happens in only two cases:
{
\setdefaultleftmargin{4mm}{}{}{}{}{}
\begin{enumerate}
    \item Unpaired \cross~instructions -- we reduce such instructions into their corresponding \move~variants (since \cross~are \move{}s with relative scheduling), and
    \item \cross~pairs with conflicting schedules (e.g., both pair sides have same priority, or instructions within a pair's side having different priorities) -- in this case, we randomly pick a valid schedule (i.e., its impact is local).
\end{enumerate}
}
This is sufficient to allow knitting on the machine.
Note that \stack~are semantically \emph{supposed} to appear with a \move, but they don’t prevent knitting since their operations lead to the same as \knit~when unpaired, and thus do not require any specific post-processing.

\section*{Additional quantitative results}

The focus of the experiments in the main paper was on assessing specific trends such as the impact of the dataset size, or the different behaviours of baseline networks, the impact of mixing data types and the ratios of these.

As can be noted, we used a standard (residual) architecture and tried to avoid over-engineering our network or its parameters.
However, we provide here results that show that we can obviously still do better by using more complex and larger networks, to the detriment of having to train for a longer time and resulting in a much larger model size.

In our baseline, we compared with a sample architecture from~\cite{Zhou16}, which we made small enough to compare with our baseline \texttt{Img2prog} implementation. Furthermore, our baseline implementations were all trained from scratch and did not make use of pre-training on any other dataset.

Here, we provide results for a much larger variant of that network, which we name \emph{Large Scene Parsing}, and makes use of pre-training on ImageNet~\cite{russakovsky2015imagenet}.
The quantitative comparison is provided in Table~\ref{tab:large}, which shows that we can achieve even better accuracy than our best current results using our \texttt{Refiner}+\texttt{Img2prog}++ combination.
However, note that this comes with a much larger model size: ours has $1.4M$ parameters\footnote{M for Million}, whereas \emph{Large Scene Parsing} has $51.4M$.
Furthermore, this requires pre-training on ImageNet with millions of images (compared to our model working with a few thousands only).

\section*{Additional qualitative results}
We present additional qualitative results obtained from several networks in Figure~\ref{fig:baselines_large}.

\figBaselinesL

\section*{Proof of Theorem 1}
We first describe the necessary definitions and lemmas to prove Theorem 1. 
We need a general way to measure the discrepancy between two distributions, which we borrow from the definition of discrepancy suggested by~\cite{mansour2009domain}.\vspace{1mm}

\begin{definition}[Discrepancy \cite{mansour2009domain}]
Let $\calH$ be a class of functions mapping from $\calX$ to $\calY$. The discrepancy between two distribution $\calD_1$ and $\calD_2$ over $\calX$ is defined as 
\begin{equation}
\mathrm{disc}_\calH(\calD_1, \calD_2) = \max\limits_{h,h' {\in} \calH} \left| \calL_{\calD_1}(h, h') - \calL_{\calD_2}(h, h')\right|.    
\end{equation}
\end{definition}
The discrepancy is symmetric and satisfies the triangle inequality, regardless of any loss function.
This can be used to compare distributions for general tasks even including regression.

The following lemma is the extension of Lemma 4 in \cite{ben2010theory} to be generalized by the above discrepancy.\vspace{1mm}

\begin{lemma}
	\label{thm:mixeddomain_gap}
	Let $h$ be a hypothesis in class $\calH$, and assume that $\calL$ is symmetric and obeys the triangle inequality. Then
	\begin{equation}
	|\calL_\alpha(h,y) - \calL_T(h,y)| \leq \alpha \left( \mathrm{disc}_\calH(\calD_S,\calD_T) + \lambda \right),
	\end{equation}
	where \hbox{$\lambda {=} \calL_{S}(h^*, y) {+} \calL_{T}(h^*,y)$}, and the ideal joint hypothesis $h^*$ is defined as \hbox{$h^*{=}\argmin_{h\in \calH} \calL_{S}(h, y) {+} \calL_{T}(h,y)$}.
\end{lemma}
\begin{proof}
The proof is based on the triangle inequality of $\calL$, and the last inequality follows the definition of the discrepancy.
\begin{align}
        &   |\calL_\alpha(h,y) - \calL_T(h,y)| \nonumber\\
=       &   \alpha |\calL_S(h,y) - \calL_T(h,y)| \nonumber\\
=       &   \alpha \left|\calL_S(h,y) - \calL_S(h^*,h) + \calL_S(h^*,h) \right. \nonumber\\
        &   -\calL_T(h^*,h) + \calL_T(h^*,h) - \calL_T(h,y) \left| \right. \nonumber\\
\leq    &   \alpha \big| \left|\calL_S(h,y) - \calL_S(h^*,h) \right| +  \nonumber\\
        &   \left|\calL_S(h^*,h) - \calL_T(h^*,h) \right| 
    + \left|\calL_T(h^*,h) - \calL_T(h,y) \right|  \big|  \nonumber\\
\leq    & \alpha \big| \calL_S(h^*,y) {+} |\calL_S(h^*,h) {-} \calL_T(h^*,h)|
    {+} \calL_T(h^*,y)   \big|  \nonumber\\
\leq    & \alpha \left( \mathrm{disc}_\calH(\calD_S, \calD_T) + \lambda \right).
\end{align}
We conclude the proof.
\end{proof}

Many types of losses satisfy the triangle inequality, \eg, the $0-1$ loss \cite{ben2010theory,crammer2008learning} and
$l_1$-norm obey the triangle inequality, and $l_p$-norm ($p>1$) obeys the pseudo triangle inequality \cite{galanti2017theory}.

Lemma~\ref{thm:mixeddomain_gap} bounds the difference between the target loss and $\alpha$-mixed loss.
In order to derive the relationship between a true expected loss and its empirical loss, we rely on the following lemma.\vspace{1mm}

\begin{lemma}[\cite{ben2010theory}]
    \label{thm:true_empricial_gap}
For a fixed hypothesis $h$, if a random labeled sample of size $m$ is generated by
drawing $\beta m$ points from $\calD_S$ and $(1 - \beta)m$ points from $\calD_T$, and labeling them according
to $y_S$ and $y_T$ respectively, then for any $\delta\in(0,1)$, with probability at least $1-\delta$ (over the choice of the samples),
\begin{equation}
\label{eq:true_empricial_gap}
    |{\hat\calL}_\alpha (h,y) - {\calL}_\alpha (h,y)| \leq \epsilon(m, \alpha, \beta, \delta),
\end{equation}
where $\epsilon(m, \alpha, \beta, \delta) = \sqrt{ \tfrac{1}{2m} \left( \tfrac{\alpha^2}{\beta} + \tfrac{(1-\alpha)^2}{1-\beta} \right) \log(\tfrac{2}{\delta}) }$.
\end{lemma}
The detail function form of $\epsilon$ will be omitted for simplicity. 
We can fix $m$, $\alpha$, $\beta$, and $\delta$ when the learning task is specified, then we can treat $\epsilon(\cdot)$ as a constant.\vspace{1mm}

\clearpage

\setcounter{theorem}{0}
\begin{theorem}
\label{thm:generalization_bound}
Let $\calH$ be a hypothesis class, and
$\calS$ be a labeled sample of
size $m$ generated by drawing $\beta m$ samples from $\calD_S$ and $(1-\beta) m$ samples from $\calD_T$ and labeling
them according to the true label $y$. 
Suppose $\calL$ is symmetric and obeys the triangle inequality.
Let $\hat h \in \calH$ be the empirical minimizer of \hbox{${\hat h} = \argmin_{h} {\hat \calL}_\alpha(h,y)$} on $\calS$ for a fixed $\alpha \in [0,1]$, 
and \hbox{$h^*_T = \argmin_h \calL_T(h,y)$} the target error minimizer.
Then, for any $\delta \in (0,1)$, with
probability at least $1 - \delta$ (over the choice of the samples), we have
\begin{equation}
\label{eq:generalization_bound}
\tfrac{1}{2}|\calL_T({\hat h},y) - {\calL}_T(h^*_T,y)| \leq  \alpha \left( \mathrm{disc}_\calH(\calD_S,\calD_T) + \lambda \right) + \epsilon ,
\end{equation}
where $\epsilon(m, \alpha, \beta, \delta) = \sqrt{ \tfrac{1}{2m} \left( \tfrac{\alpha^2}{\beta} + \tfrac{(1-\alpha)^2}{1-\beta} \right) \log(\tfrac{2}{\delta}) }$, and \hbox{$\lambda {=} \min_{h\in\calH} \calL_{S}(h, y) {+} \calL_{T}(h,y)$}.
\end{theorem}

\begin{proof}
We use Lemmas~\ref{thm:mixeddomain_gap} and \ref{thm:true_empricial_gap} for the bound derivation with their associated assumptions.
\begin{align}
& \calL_T({\hat h},y) \nonumber\\
& \leq \calL_\alpha({\hat h},y) + \alpha \left( \mathrm{disc}_\calH(\calD_S,\calD_T) + \lambda \right),\\
& \hspace{0.6\linewidth} \textrm{(By Lemma~\ref{thm:mixeddomain_gap})} \nonumber\\
& \leq {\hat\calL}_\alpha({\hat h},y) + \alpha \left( \mathrm{disc}_\calH(\calD_S,\calD_T) + \lambda \right) + \epsilon, \\
& \hspace{0.6\linewidth} \textrm{(By Lemma~\ref{thm:true_empricial_gap})} \nonumber\\
& \leq {\hat\calL}_\alpha(h^*_T,y) + \alpha \left( \mathrm{disc}_\calH(\calD_S,\calD_T) + \lambda \right) + \epsilon, \\
& \hspace{0.6\linewidth} ({\hat h} = \argmin_{h\in\calH} {\hat \calL}_\alpha(h)) \nonumber\\
& \leq {\calL}_\alpha(h^*_T,y) + \alpha \left( \mathrm{disc}_\calH(\calD_S,\calD_T) + \lambda \right) + 2\epsilon, \\
& \hspace{0.6\linewidth} \textrm{(By Lemma~\ref{thm:true_empricial_gap})} \nonumber\\
& \leq {\calL}_T(h^*_T,y) + 2\alpha \left( \mathrm{disc}_\calH(\calD_S,\calD_T) + \lambda \right) + 2\epsilon, \\
& \hspace{0.6\linewidth} \textrm{(By Lemma~\ref{thm:mixeddomain_gap})} \nonumber
\end{align}
which concludes the proof. 
\end{proof}

Theorem 1 does not have unnecessary dependencies for our purpose, which are used in~\cite{ben2010theory} such as unsupervised data and the restriction of the model type to finite  VC-dimensions.

\end{document}